\documentclass[a4paper,12pt]{article}

\title{Closed-form Filtering for Non-linear Systems}
\author{Théophile Cantelobre$^{1}$\\Carlo Ciliberto$^2$\\Benjamin Guedj$^{2,3}$\\Alessandro Rudi$^{1}$ \\ \small $^1$ Inria/PSL Research University, Paris, France\\\small $^2$ AI Centre, Dept. of Computer Science, University College London, London, UK\\\small $^3$ Inria Lille - Nord Europe and Inria London, France}

\usepackage[T1]{fontenc}
\usepackage{graphicx}
\usepackage{amssymb}
\usepackage{amsmath}
\usepackage[font=footnotesize]{caption}
\usepackage{xcolor}
\definecolor{dgreen}{rgb}{0.00,0.49,0.00}
\definecolor{dblue}{rgb}{0,0.08,0.75}
\usepackage{hyperref}
\hypersetup{
    colorlinks = true,
    citecolor=dgreen,
    urlcolor=dblue,
    linkcolor=dblue
}
\usepackage[nameinlink,capitalize]{cleveref}
\usepackage{natbib}
\bibpunct{(}{)}{;}{a}{,}{,}
\usepackage[ruled]{algorithm2e}

\usepackage{stackengine}
\usepackage{scalerel}
\usepackage{amsthm}

\usepackage{thmtools, thm-restate}
\usepackage{macro}

\newcommand{\mixing}{\sigma}

\begin{document}

\maketitle
\begin{abstract}%
Sequential Bayesian Filtering aims to estimate the current state distribution of a Hidden Markov Model, given the past observations. The problem is well-known to be intractable for most application domains, except in notable cases such as the tabular setting or for linear dynamical systems with gaussian noise. In this work, we propose a new class of filters based on Gaussian PSD Models, which offer several advantages in terms of density approximation and computational efficiency. We show that filtering can be efficiently performed in closed form when transitions and observations are Gaussian PSD Models. When the transition and observations are approximated by Gaussian PSD Models, we show that our proposed estimator enjoys strong theoretical guarantees, with estimation error that depends on the quality of the approximation and is adaptive to the regularity of the transition probabilities.  In particular, we identify regimes in which our proposed filter attains a TV $\epsilon$-error with memory and computational complexity of $O(\epsilon^{-1})$ and $O(\epsilon^{-3/2})$ respectively, including the offline learning step, in contrast to the $O(\epsilon^{-2})$ complexity of sampling methods such as particle filtering.
\end{abstract}

\section{Introduction}

Sequential Bayesian Filtering is the task of inferring the distribution of unobserved variable $X_T$ from observations $Y_1, \ldots, Y_T$ where $(X_t, Y_t)_{n \geq 0}$ is a Hidden Markov Model. The distribution of $X_T$ given $Y_1, \ldots, Y_T$ is known as the filtering distribution (or optimal filter), denoted $\pi_T^\nu$ where $\nu\in\mathcal P(E)$ is the distribution of $X_0$ (or prior knowledge on $X_0$ more generally). In a Hidden Markov Model, the conditional distributions of $X_n$ given $X_{n-1}$ and $Y_n$ given $X_n$ are described by transition kernels $Q$ and $G$ respectively. In particular, $Q(u,x)$ identifies the conditional probability of transitioning to the state $X_{n}=x$ at time $n$ given the fact that the system was in the state $X_{n-1}=u$ at time $n-1$, while $G(x,y)$ is the probability of observing $Y_{n} = y$ given the fact that the system is in the state $X_{n} = x$. If one has perfect knowledge of the initial distribution $\pi_0$, the transition and observation kernels $Q, G$, then the filtering distribution can be computed recursively by applying Bayes rule:
\begin{align}\label{eq:bayes-0}
\pi_n(dx) = \frac{\int Q(u, dx)G(x, y_n)\pi_{n-1}(du)}{\int\int Q(u, dx)G(x, y_n)\pi_{n-1}(du)}
\end{align}
The recursive expression of the filtering distribution can be seen as a combination of two steps: the prediction of $X_{n+1}$ given belief on $X_n$ given past observations, then the correction of this prediction based on the observation $Y_{n+1}$ received. For complete introduction to Hidden Markov Models and inference, we refer the reader to \cite{cappehmm}. 

Two quintessential problems in filtering are studying the stability and robustness of the optimal filter, and computing an approximation of it in practice.

The stability of the optimal filter is related the robustness of the sequence $\pi_n^\nu$ with respect to the initial  $\nu$, since this distribution is unknown in most application settings and estimates/priors need to be used in practice. The optimal filter is stable when its dependence on the initial distribution decreases as more observations are provided (also known as ``forgetting''). More formally, a filter is stable when the distance between $\pi_n^\nu$ and $\pi_n^\mu$ goes to zero as $n\to\infty$ for any two initial distributions $\nu$ and $\mu$. This problem has attracted considerable interest since the first contributions of \cite{ocone} and \cite{kunita}. Some reference relevant to this work include \cite{oudjane,mcdonald2020,legland99,decastro2017,mitrophanov-hmm-stability-2005}. A modern review of the literature and different approaches can be found in \cite{kim2022duality}.

In general, the iteration in \cref{eq:bayes-0} is intractable. Two exceptions are when the state-space in finite or when the state-space is continuous but the transition kernels are Gaussian Linear Conditional Distributions. In the former case, the algorithm is known as the \textit{forward algorithm}. In the latter, the algorithm is the well-known Kalman filter \cite{kalman-bucy}. The Kalman filter is known to compute the recursion \cref{eq:bayes-0} exactly where the Hidden Markov Model has linear dynamics and observations and independent Gaussian noise. 
If the state-space is finite, the Baum-Welch algorithm can compute filtering and smoothing distributions in closed-form using the forward-backward approach. 

But many real-world systems do not have linear dynamics nor Gaussian uncertainty. On one hand, many algorithms have been devised to handle variations on the Kalman filter's assumptions, including the Extended Kalman Filter or the Unscented Filter \citep{ukf} (see also \cite{sarkka}). These models approximate the state variable as a Gaussian, which excludes many systems where, for instance, multi-modality is present. On the other, Sequential Monte Carlo algorithms such as the Particle filter and variants were developed. These algorithms approximate the marginalization step above using sampling, and can handle multi-modality. This family of methods has strong theoretical guarantees (though under arguably stringent conditions), but do not give a closed-form expression of the approximate distribution and are known to be difficult to turn and prohibitively costly for online applications.

Gaussian PSD Models were introduced in \cite{ciliberto2021} as models for probability distributions. Gaussian PSD Models are a special case of the family of models proposed in \cite{ulysse-non-negative}. They generalize Gaussian Mixture Models by allowing for negative coefficients in the mixture. As originally highlighted in \cite{ciliberto2021}, Gaussian PSD Models enjoy appealing properties for applications involving Bayesian inference, with filtering as a special case: 1) They have optimal approximation guarantees with respect to a large family of probability densities 2) products and marginals of Gaussian PSD models can be efficiently computed in closed-form.

\subsection*{Approach \& Contributions}

In this paper, we study the problem of performing the iteration in \cref{eq:bayes-0} when knowledge of the transition $Q$ and observation $G$ probabilities are unknown and only an approximation in terms of two corresponding Gaussian PSD Models $\hat Q$ and $\hat G$ is available. We introduce a new algorithm to derive an estimator $\hat \pi_n(dx)$ and study its relation with the true $\pi_n(dx)$. The proposed estimator extends previous filtering strategies, such as the Kalman filter and offers strong theoretical guarantees on a large family of application settings. 

Our main contributions are: 
\begin{enumerate}
    \item A novel algorithm to tackle Sequential Bayesian Filtering, which recovers previously proposed estimators and can be applied to any filtering problem where the transition kernels admit a smooth density.

    \item We show that the proposed estimator is both stable and robust with respect to a large family of application settings. These theoretical properties are adaptive to the regularity properties of the Hidden Markov Model.

    \item The computational and space complexity of the proposed algorithm depends on the regularity of the transition kernels. For very regular kernels (e.g. infinitely differentiable) the algorithm has a computational complexity that is smaller than, for example, particle filtering.

\end{enumerate}

Our paper is organized as follows: in \cref{sec:gaussian-psd-models}, we describe Gaussian PSD Models and their properties, in particular their stability with respect to probabilistic operations. In \cref{sec:learning-gaussian-psd}, we devise an algorithm for learning Gaussian PSD Models from function evaluations and prove that optimal estimation rates are attained for smooth targets. In \cref{sec:psd-filter-root}, we introduce \textsc{PSDFilter}, an approximate filtering algorithm which plugs $\hat Q$ and $\hat G$ in the iteration above. We prove this algorithm is robust to the choice of initial distribution and to the approximation error in $Q$ and $G$. Sketches of the proofs of our main theorems \cref{theorem:learning} and \cref{theorem:bound-diagonal} are presented in \cref{sec:sketches}. Finally, in \cref{sec:generalized-psd-models}, we generalize Gaussian PSD Models to allow for a richer class of approximators, while retaining most of the desirable properties of Gaussian PSD Models.

\paragraph{Notation} We denote $E=(-1, 1)^d$ and $F = (-1, 1)^{d^\prime}$ the state and observation space. $\mathcal P(E)$ is the set of probability measures on $E$ and $\mathcal M_+(E)$ the set of finite, positive measures on $E$. We assume that all measures admit a density with respect to the Lebesgue measure and use the abuse of notation $\mu(dx)=\mu(x)dx$. $\beta>0$ is a smoothness parameter, $\mixing$ is a mixing parameter for kernels and $\epsilon$ is the accuracy when doing function approximation. Denote $\mathbb R^d_+$ the set of vectors in $\mathbb R^d$ with all positive components and $\mathcal S^+(\mathbb R^M)$ the set of positive definite matrices of size $M$.

\section{Gaussian PSD Models}\label{sec:gaussian-psd-models}

Gaussian PSD Models, introduced in \cite{ciliberto2021}, is a family of models for non-negative functions and, in particular, probability densities specializing the PSD Models from \cite{ulysse-non-negative}. They are non-negative everywhere, admit a linear parametrization and can be learned from samples and function evaluations. They are characterized by a linear combination of kernels, with weights chosen such that the function is non-negative. In this section, we recall the definition of Gaussian PSD Models of \cite{ciliberto2021}, show that they extend most well-established probability models and present how to perform operations such as multiplication or marginalization. 

\begin{definition}\label{def:psd-model}
   A Gaussian PSD Model of order $M$ is a function $f: \mathbb R^d \to \mathbb R$ which can be written:
   \begin{align}\label{eq:psd-model-def}
      f(x) = \sum_{i=1}^M\sum_{j=1}^M A_{ij}k_\eta(x, x_i)k_\eta(x, x_j)
   \end{align}
   where $\eta \in \mathbb R^d_+$ is the precision vector, $X = (x_i)_{1\leq i \leq M}\in \mathbb R^{M\times d}$ are the anchor points and $A\in\mathcal S^+(\mathbb R^d)$ is the weight matrix. Such a function is denoted $f(x; A, X, \eta)$ (or $f(x;\theta)$ for shorthand).
\end{definition}
Adopting the point of view in \cite{ulysse-non-negative}, Gaussian PSD Models can equivalently be defined as functions of the form $f(x) = \Phi_\eta(x)^\top A\Phi(x)$ where $\Phi_\eta$ is defined as $\Phi_\eta(x)= (k_\eta(x, x_1) \ldots k_\eta(x, x_M))^\top \in\mathbb R^M$.
When defined on the product of Euclidean spaces $\mathbb R^d \times \mathbb R ^{d^\prime}$ with anchor points $[X, Y]$ (the row-wise concatenation of $X$ and $Y$) and precision vector $\eta = (\eta, \eta^\prime)$ (column-wise concatenation of $\eta_1$ and $\eta_2$), we denote denote the model $f(x, y ; A, [X, Y], (\eta_1, \eta_2))$. This split notation is justified by the fact that $k_\eta((x, y), (u, v)) = k_{\eta_1}(x, u)k_{\eta_2}(y, v)$.

\begin{example}[Gaussian Mixture Model]\label{ex:mixtures} Let $p(x) = \sum_{k=1}^M\alpha_kp(x|\mu_k, \eta)$ where $\mu_k\in\mathbb R^d$, $\eta_k\in\mathbb R^d_+$, and $\alpha \in\mathbb R^d_+$ with $\sum_{k=1}^M\alpha_k = 1$ and $p(x|\mu_k, \eta)$ is the Gaussian density with mean $\mu$ and precision vector $\eta$. $p$ is known as a Gaussian Mixture Model. $p$ can be written as a Gaussian PSD Model of order $M$ $f(x ; A, X, \eta / 2)$ with $A=\textrm{diag}(a)$ and $X = (\mu_1 \ldots \mu_M)^\top $.
\end{example}

\begin{example}[Squared linear Gaussian model]\label{ex:sq-linear-model}
Let $g(x) = w^\top \Phi_\eta(x)$ where $w\in\mathbb R^d$ and $\Phi_\eta(x)= (k_\eta(x, x_1), \ldots, k_\eta(x, x_M)^\top $. Then, $f=g^2$ can be written as a Gaussian PSD Model of order $M$ with $A=ww^\top $ and $X = (x_1, \ldots, x_M)^\top $. Indeed, 
\begin{align}
    f(x) = (w^\top \Phi_\eta(x))^2 = w^\top \Phi_\eta(x)w^\top \Phi_\eta(x) = \Phi_\eta(x)^\top ww^\top \Phi_\eta(x).
\end{align}
\end{example}

\noindent As pointed out in the introduction of this section, because $k_\eta(x, u)k_\eta(x, v)\propto k_{2\eta}(x, \frac{u+v}{2})$, a Gaussian PSD Model can be seen as a linear combination of Gaussians. It is important to note that the coefficients of the components can be non-negative, which makes them much more expressive then Mixture models. Consider for instance $f(x) = (e^{-(x-2)^2} - e^{-(x-3)^2})^2$ which is clearly non-negative and can be written as a Gaussian PSD Model but not as a Mixture model.

\subsection{Operations on Gaussian PSD Models}
Gaussian PSD Models are compatible with operations on probabilistic models such as integration, partial evaluation, product and marginalization. The operations are summarized in \cref{prop:ops-diagonal} and the algorithms, based on kernel evaluations and matrix-vector products are detailed in \cite{ciliberto2021}.

\begin{proposition}[Closed form operations for Gaussian PSD Models]\label{prop:ops-diagonal}
    Let $f(x, y; \theta_1)$ and $g(y, z; \theta_2)$ be two Gaussian PSD Models of order $M_1$ and $M_2$ respectively, as in \cref{def:psd-model}. Then there exist some algorithms $\integralpsd$, $\partialpsd$, $\productpsd$, $\marginalpsd$ such that
    \begin{itemize}
        \item \textbf{ Integral over $\mathbb R^d$ or over a hypercube} $\int f(x, y; \theta_1)dxdy$ can be computed exactly and in closed form by the algorithm $\integralpsd(\theta_1)$ with a computational cost of $O(M_1^2 d)$.
        \item \textbf{ Partial evaluation}$f(x, y_0; \theta_1) = h(x; \theta^\prime)$ is a Gaussian PSD Model of order at most $M_1$ and $\theta^\prime$ can be computed exactly and in closed form by the algorithm $\partialpsd(y_0, \theta_1)$ with a computational cost of $O(M_1^2 d)$
        \item \textbf{ Product} $f(x, y; \theta_1)g(y, z; \theta_2)=h(x, y, z; \theta^\prime)$ is a Gaussian PSD Model of order at most $M_1\times M_2$ and $\theta^\prime$  can be computed exactly and in closed form by the algorithm $\productpsd(\theta_1, \theta_2)$ with a computational cost of $O(M_1^2 M_2^2 d)$.
        \item \textbf{ Marginalization} $\int f(x, y; \theta_1)dy=h(x; \theta^\prime)$ is a Gaussian PSD Model of order at most $M_1$ and $\theta^\prime$ can be computed exactly and in closed form by the algorithm $\marginalpsd(y, \theta_1)$ with a computational cost of $O(M_1^2 d)$.
    \end{itemize}
\end{proposition}
The proof of the proposition above can be found in \citep[Appendix F]{ciliberto2021}.
Note that a Markov transition $g(x) = \int Q(u, x)f(u)du$, when $Q$ and $f$ are Gaussian PSD models, can be decomposed in terms of product and marginalization and computed in closed form, with $g$ again a Gaussian PSD Model. More importantly,  $g$ is of order $M_1$ (instead of the naïve $M_1\times M_2$ according to \cref{prop:ops-diagonal}). This is summarized in the following proposition

\begin{proposition}[Constant order for Markov transition]\label{prop:markov-step}
   If $Q(u, x; \theta_Q$ and $f(u; \theta_f)$ are two Gaussian PSD Models of order $M_1$ and $M_2$ respectively, then $g(x) = \int Q(u, x;\theta_Q)f(u;\theta_f)du$, computed via $\productpsd$ and $\marginalpsd$ is a Gaussian PSD Model of order $M_1$. 
\end{proposition}
The proof of \cref{prop:markov-step} can be found in \citep[Appendix F.5]{ciliberto2021}.

\section{Learning transition and observation kernels with Gaussian PSD Models}\label{sec:learning-gaussian-psd}
In this section, we show that General Gaussian Models can be used to efficiently approximate smooth, non-negative functions in $L^\infty$ using function evaluations. Let $\Omega = (-1, 1)^d$ and $f: \Omega \to \mathbb R$ be the target function. We assume we can evaluate $f(x)$ at any point $x\in\Omega$. We assume that $f$ is the sum of squared $\beta$-smooth functions. Formally, we introduce:

\begin{assumption}[Smooth sum-of-squares assumption]\label{assumption:target_function}
    There exist $q \in\mathbb N$ and $\beta \geq 0$ and $f_i \in W^\beta_2(\Omega) \cap L^\infty(\Omega)$ such that $f(x) = \sum_{i=1}^q f_i(x)^2$.
\end{assumption}
\cref{assumption:target_function} is verified for most continuous dynamical models of interest. For instance, any transition kernel $Q(x, y) \propto e^{-\Vert \Sigma^{-1/2}(y - h(x)\Vert^2}$ verifies the assumption. \cite{ciliberto2021} provides a list of sufficient conditions, we recall in \cref{prop:sufficient-conditions-target}.

\begin{proposition}[Generality of \cref{assumption:target_function}, Prop. 5 in \cite{ciliberto2021}]\label{prop:sufficient-conditions-target}
A function $f$ satisfies \cref{assumption:target_function} on $\Omega = (-1, 1)^d$ as soon as:
\begin{itemize}
    \item $f$ is a probability density and $f\in W_2^\beta(\Omega)\cap L^\infty(\Omega)$, and strictly positive on $[-1, 1]^d$ ;
    \item $f$ is an exponential model $f(x) = e^{-v(x)}$ with $v\in W_2^\beta(\Omega) \cap L^\infty(\Omega)$ ;
    \item $f$ is a mixture of models from (b) ;
    \item $f$ is $\beta+2$-times differentiable on $[-1, 1]^d$, with a finite set of zeroes all in $(-1, 1)^d$, and a positive definite Hessian in each zero.
\end{itemize}
\end{proposition}
In \cref{sec:optimization}, we introduce the optimization problem we solve to learn $f$ and present the learning algorithm. In \cref{sec:learning}, we prove that the obtained estimator $\hat f$ is 

\subsection{Learning algorithm}\label{sec:optimization}

In \cref{sec:gaussian-psd-models}, we showed that the square of any Gaussian Linear Model is a Gaussian PSD Model and its weight matrix is of rank $1$ and given by $A =aa^\top $ where $a$ is the weight vector of the Gaussian Linear Model. We use this insight to efficiently approximate a smooth sum-of-squares function $f$ with a Gaussian PSD Model $\hat f$. This insight was first published by \cite{sampling-ulysse} for probability densities. A full-rank estimator $\hat f$ can be also learned, by solving a Semi-Definite Programming problem using e.g. Newton's method.


In this work we propose to approximate $f$ using $\hat f= \hat g^2$ where $\hat g$ is a Gaussian Linear Model learned on $g$. Denoting the Linear Gaussian Model $\hat g(x ; a, \eta, \tilde X) = \sum_{i=1}^M a_ik_\eta(x, \tilde x_i)$ where $a\in\mathbb R^M$, $\eta \in \mathbb R_+^d$ and $\tilde X \in \mathbb R^{M\times d}$, we introduce the optimization problem used to learn $\hat a$ from data points $X\in\mathbb R^{n\times d}$:
\begin{align}\label{eq:learning-problem}
\min_{a \in\mathbb R^M} \frac{1}{n}\sum_{k=1}^n \vert \sqrt{f(x_k)} - \hat g(x_k ; a, \eta, \tilde X)\vert^2+ \lambda \, a^\top Ka 
\end{align}
where $K$ is described by $K_{ij}=k_\eta(\tilde x_i, \tilde x_j)$. 

In \cref{sec:proof-learning}, we cast \cref{eq:learning-problem} as a kernel ridge regression problem, which can be efficiently solved for large values of $n$ and $M$ in $O(n\sqrt{n})$ time using approximate kernel methods such as in \cite{falkon}.

\begin{algorithm}[ht!]
\caption{\textsc{LearnRankOne} algorithm}\label{alg:learn}
\KwData{$f(x)$, $M$, $N$, $\eta$, $\lambda$}
$X \gets \textsc{UniformSample}(\Omega, N)$\;  \BlankLine
$Y \gets (\sqrt{f(x_i)} \vert x_i \in X)$\;\BlankLine
$\tilde X\gets \textsc{UniformSample}(\Omega, M)$\;  \BlankLine
$\hat a \gets \textsc{KernelRidgeRegression}(k_\eta, \tilde X, X, Y, \lambda)$ \; \BlankLine
$\hat f(x) \gets \textsc{GaussianPSDModel}(aa^\top , X, \eta)$\;\BlankLine
\KwResult{$\hat f(x)$}
\end{algorithm}

\subsection{Learning rates}\label{sec:learning}
The Gaussian PSD Model $\hat f$ obtained from \cref{alg:learn} using function evaluations at uniformly sampled training points $X$ approximates $f$ in $L^\infty(\Omega)$ with optimal learning rates for $L^\infty$ norm \citep{wendland2004scattered}, if $f$ is a $\beta$-smooth and bounded density, as formalized by \cref{assumption:target_function}.

We build on the results in \cite{ciliberto2021} and \cite{sampling-ulysse}. The former studies the $L^2$ convergence of $\hat f$ to $f$ when the training set is sampled from the target density, using the full-rank counterpart to \cref{alg:learn}. The latter studies convergence in Hellinger distance using \cref{alg:learn}. Both works use the insights of \cite{less-is-more}.

\cref{alg:learn} finds a solution to \cref{eq:learning-problem} in the reproducing kernel Hilbert space $\mathcal H_\eta$ associated to $k_\eta$ where $\eta$ is chosen as a function of the desired precision $\epsilon$. \cref{theorem:learning} proves that $\hat f$ converges to $f$ in $L^\infty$, with optimal rates \citep{wendland2004scattered}. In particular, to learn $f$ uniformly to precision $\epsilon$, $N\approx\epsilon^{-\frac{2d}{2\beta -d}}$ function evaluations and a model with $M \approx N$ anchor points suffices. 

\begin{theorem}\label{theorem:learning}
Let $\beta > d/2$ and $\theta^{-1} < 1 + 2\beta/d$.  Let $f:\Omega \to \mathbb R$ such that $f$ verifies \cref{assumption:target_function}. Set $M \geq C' (\log(\frac{1}{\epsilon}))^d\log(\frac{1}{\delta\epsilon})\epsilon^{-d/\beta}$ and $n \geq C' \epsilon^{-2 d/\beta} \log \frac{1}{\delta}$. Consider the set of anchor points $\tilde X \in \mathbb R^{M \times d}$ and the set of training points $X \in \mathbb R^{n \times d}$ sampled independently and uniformly from $\Omega = (-1, 1)^d$. Let $\epsilon\leq \epsilon_0$. Let $\eta = \epsilon^{-2/\beta} \, 1_d$ and $\mathcal H_\eta$ the RKHS associated to $k_\eta$. Let $\hat a$ solution to the kernel ridge regression problem defined in \cref{eq:learning-problem}. We denote $\hat g(\cdot) = \hat g(\cdot ~; \hat a, \tilde X, \eta)$ the estimator of $g$ and $\hat f = \hat g^2$. With probability at least $1 - 3\delta$,
\begin{align}
    \Vert \hat f - f \Vert_{L^\infty(\Omega)}\leq C \Vert \sqrt{f}\Vert_{W^\beta_2(\Omega)}^2\epsilon^{1-\frac{d}{2\beta}}
\end{align}
where $C, C'$ are constants depending only on $\beta, d$ and independent of $f$ and $\epsilon$.
\end{theorem}
The proof of \cref{theorem:learning} can be found in \cref{sec:proof-learning}, and uses arguments from \cite{sampling-ulysse}. A sketch is given in \cref{sec:sketches}.

\section{Gaussian PSD Models for filtering}\label{sec:psd-filter-root}
Now that we know how to efficiently carry out closed-form filtering operations using Gaussian PSD Models and learn good approximations of non-negative functions using this family of models, we can compute an approximation of the filtering distribution.

\subsection{Setting}
Consider two discrete Markov Chains $(X_n)\in E^\mathbb N$ - the hidden state chain -  and $(Y_n)\in F^\mathbb N$ - the observations. We assume that $(X_t, Y_t)$ has a Hidden Markov Model structure described by $(\nu, Q, G)$ where $\nu\in\mathcal P(E)$ is an initial distribution, $Q:E\times E \to \mathbb R$ is a Markov kernel, and $G: E\times F \to \mathbb R$ is a transition kernel. Formally, this can be summarized as:
(1) the law of $X_n$ is fully-determined by the knowledge of $X_{n-1}$, i.e. $\mathbb P(X_n\in dx | X_{n-1}=x)=Q(x, dx)$ ;
(2) the law of $Y_n$ is fully-determined by the knowledge of $X_n$; i.e. $\mathbb P(Y_n\in dy |X_n=x) = G(x, dy)$ ;
(3) the law of $X_0$ is given by $\nu$. In particular, we have the usual Markovian structure: $X_n \independent X_l | X_{k}$ and $Y_n \independent Y_l | X_{k}$ for any $n > k > l$.

The goal of filtering is to compute the distribution of $X_n$ conditionally on past observation $Y_1, \ldots, Y_n$. We denote this distribution $\pi_n^\nu(z_{1:n}, dx)$ where $\nu$ is initial distribution and $z_{1:n}$ are the observations (not necessarily taken from the chain $(Y_n)$. This distribution is known as the filtering distribution or optimal filter. When clear from context, we drop the dependence of $z_{1:n}$ and $\nu$.

Importantly, $\pi_n(dx)$ can be computed recursively using $Q$ and $G$ and beginning from $\nu$:
\begin{align}\label{eq:iteration}
\pi_n = \bar{R}_n(\pi_{n-1}), ~~\textrm{where}~~ \bar{R}_n(\mu)(\cdot) := \frac{R_n \mu \, (\cdot)}{R_n\mu\,(E)} ~~\textrm{and}~~  R_n \mu \, (\cdot) := \int Q(u, \cdot)G(x, y_n) d\mu(u)
\end{align}
which recovers \cref{eq:bayes-0}. Computing $\bar R_n(\pi_{n-1})$ is difficult in most circumstances since one must be able to compute products and marginals on probability distributions. Two notable exceptions include the Conditional Linear Gaussian Model (which corresponds to the Kalman filter) and when $E$ is finite (which corresponds to the Baum-Welch algorithm).

\subsection{PSD filter}\label{sec:psdfilter}
To overcome these difficulties, we approximate $Q$ and $G$ from evaluations using Gaussian PSD Models then compute iteration \cref{eq:iteration} with these approximate kernels.

Given a sequence of observations $(z_k)_{k\geq 1}$, we define $\hat R_k(u, x) = \hat Q(u, x)\hat G(x, y_k)$ analogously to $R_n$. The non-linear transformation $\barhat{R}_k$ is defined for any positive, finite measure $\mu$ by 
\begin{align}\label{eq:approx-iteration}
    \barhat{R}_k (\mu) = \frac{\hat R_k\mu}{\hat R_k\mu\,(E)}.
\end{align}
and summarized in \cref{alg:fullfilter}.
Note that at each step, $\hat \pi_k$ is a valid, normalized density. However, $\hat Q$ and $\hat G$ are not properly normalized, i.e. $\hat G(u, F) =1$ is not guaranteed.
As long as the initial distribution $\hat \pi_0$ is a General PSD Model and $\hat G$ and $\hat Q$ are valid Gaussian PSD Models, $\hat \pi_k$ is a Gaussian PSD Model for all $k\geq 0$ and moreover is a valid density. Importantly, \cref{theorem:algorithm} shows that the order of $\hat \pi_k$ is constant for $k\geq 1$ and equal to the product of the orders of $\hat Q$ and $\hat G$ as shown in the following corollary which follows directly from \cref{prop:markov-step}.
\begin{corollary}[$\hat{\pi}_n$ has constant order for any $n$] \label{theorem:algorithm}
    Let $\hat \pi_0(x)$ a Gaussian PSD Model on $E$ of order $M_0$, $\hat Q(u, x)$ a Gaussian PSD Model on $E\times E$ of order $M_Q$ and $\hat G(x, y)$ a Gaussian PSD Model on $E\times E$ of order $M_G$. Let $(z_n)\in F^\mathbb N$. Let $(\hat \pi_n)_{n\in\mathbb N}$ the sequence of functions defined by the recursion \cref{eq:approx-iteration}. Then, for any $n\geq 0$, $\hat \pi_n$ is a Gaussian PSD Model of order at most $M_Q \times M_G$ and it is computed by \cref{alg:fullfilter}.
\end{corollary}

\begin{algorithm}[ht!]
\caption{PSDFilter algorithm}\label{alg:fullfilter}
\KwData{$z_1, \ldots, z_T$, $\hat  \pi_0$, $\hat Q$, $\hat G$}
\For{$k=1,\dots, T$}{%
    $\beta \gets \productpsd(\hat\pi_{k-1}, \hat Q)$\;  \BlankLine
    $\hat Q \hat\pi_{k-1}(\cdot) \gets \marginalpsd\left(\beta(u, \cdot), [u]\right)$\;  \BlankLine
    $\hat G_k(\cdot) \gets \partialpsd(\hat G(\cdot, y), y:=y_k)$\;  \BlankLine
    $\tilde \pi_k \gets \productpsd\left(\hat Q \hat\pi_{k-1}, \hat G_k\right)$\;  \BlankLine
    $Z\gets \integralpsd(\tilde\pi)$\;  \BlankLine
    $\hat\pi \gets \tilde\pi / Z$\;\BlankLine
}
\KwResult{$\hat \pi_1, \ldots, \hat \pi_T$}
\end{algorithm}

\subsection{Gaussian PSD Filter Stability and Robustness}\label{sec:theory}
In this section, we show that the algorithm introduced in \cref{sec:psdfilter} estimates a distribution which is close to the true filtering distribution in total variation distance. \cref{theorem:bound-diagonal} combines stability properties of the optimal filter with robustness of the iteration \cref{eq:approx-iteration} to the use of $\hat Q$ and $\hat G$ in place of $Q$ and $G$. We introduce the following assumption. 

\begin{assumption}[$R_n$ is mixing]\label{assumption:mixing}
    There exists $1 >\mixing > 0$ and a probability density $\xi\in\mathcal P(E)$ such that for any $n \in\N$, $R_n$ is $\mixing$-$\xi$-mixing, i.e. for any $u, x \in E \times E$,
    \begin{align}
        \mixing \xi(x) \leq R_n(u, x) \leq \frac{1}{\mixing}\xi(x).
    \end{align}
\end{assumption}
\cref{assumption:mixing} is a classical assumption for the study of filtering \citep{cappehmm}. We are ready to state the main theorem of this work.

\begin{theorem}[PSD filter robustness and stability]\label{theorem:bound-diagonal}
Assume that $Q$ and $G$ verify \cref{assumption:target_function,assumption:mixing}. Let $\varepsilon > 0$. 
When $\hat{G}, \hat{Q}$ are learned using \cref{alg:learn} and $M, N \in \N$ are chosen such that
$$M \geq C' (\log\left(1/\varepsilon\right))^D\log\left(1/\varepsilon\delta\right)\varepsilon^{-\frac{D}{\beta - D/2}}, \quad n \geq C' \varepsilon^{-\frac{D}{\beta - D/2}}\log (1/\delta)$$ 
then with probability at least $1-6\delta$, the following holds: for any $k \in \N$, 
\begin{align}
    \Vert \pi_k - \hat\pi_k\Vert_{TV} ~~\leq~~ \frac{C}{\mixing^2}\left(\frac{1-\mixing^2}{1 + \mixing^2}\right)^{k-1}\|\pi_0 - \hat{\pi}_0\|_{TV} ~~+~~ \frac{\varepsilon}{\sigma},
\end{align}
where $\hat{\pi}_1,\dots, \hat{\pi}_k$ are computed using \cref{alg:fullfilter}, moreover $C=\frac{2}{\log 3}$ and $C'$ depends only on $\|\sqrt{Q}\|_{W^\beta_2}, \|\sqrt{G}\|_{W^\beta_2}, \beta, d$.
\end{theorem}
\cref{theorem:bound-diagonal} is proven in \cref{sec:proof-bound} and we include a sketch of the proof in \cref{sec:proof-bound-sketch}.

The theorem above shows that the distance between the probability $\hat \pi_k$ resulting from our algorithm and the optimal one $\pi_k$ corresponding to $k$ steps of sequential Bayesian filtering with the true $Q, G$ is bounded by two terms: (1) 
the first term accounts for stability and goes to zero exponentially fast in the number of steps $k$ and depends on how close we choose $\hat{\pi}_0$ with respect to $\pi_0$; (2) the second terms accounts for robustness, it does not increase in $k$ and we can make it arbitrarily small by learning more precise $\hat{Q}, \hat{G}$, i.e. by increasing $M, n$. Assuming that $\hat{\pi}_0 = \pi_0$ the proposed algorithm \cref{alg:fullfilter} for any $k \in\N$ achieves a  
an error $\|\pi_k - \hat{\pi}_k\|_{TV} \leq \varepsilon/\sigma$ producing a Gaussian PSD model $\hat{\pi}_k$ that satisfies
$$ \hat{\pi}_k ~~\textrm{of order}~~  O(\varepsilon^{-\frac{2D}{\beta - D/2}}), ~~\textrm{and} ~~ O(\varepsilon^{-\frac{6D}{\beta - D/2}}),$$
for \cref{alg:fullfilter}. The cost of learning $\hat{Q}, \hat{G}$ via \cref{alg:learn} is paid only once at the beginning and is of $O(\varepsilon^{-\frac{1.5D}{\beta - D/2}})$, if we use fast algorithms as the ones recalled in \cref{sec:learning}. Note that the proposed approach is adaptive to the regularity of the kernels $Q$ and $G$. In particular, if they are differentiable many times, i.e. $\beta \geq 4.5 D$, then the order of $\hat{\pi}_k$ becomes only  $O(\varepsilon^{-1/2})$, i.e. 
$$
\textrm{memory cost} ~~ O(\varepsilon^{-1}), \quad \textrm{computational cost} ~~ O(\varepsilon^{-1.5}).
$$ 
This is remarkable since, for example, particle filter methods are bound to a computational complexity that cannot be smaller than $\varepsilon^{-2}$ since they have to approximate an integral via sampling \citep{oudjane}.


\section{Sketch of the proof of \cref{theorem:learning} and \cref{theorem:bound-diagonal}}\label{sec:sketches}
In this section, we give the main arguments for the proof of \cref{theorem:learning}. The proof can be found in \cref{sec:proof-learning}. We denote $f$ the target function and $g = \sqrt{g}$, and $\hat g$ the solution to \cref{eq:learning-problem} and $\hat f = \hat g^2$.

As shown in \cref{sec:proof-learning}, by opening the proof of Proposition 11 in \cite{sampling-ulysse}, we can obtain bounds on $\Vert g - \hat g \Vert_{L^2(\Omega)}$ is $L^2(\Omega)$ and (essentially) $\Vert \hat g - g\Vert_{W^m_2(\Omega)}$ as a function of $\Vert g\Vert_{W^\beta_2(\Omega)\cap L^\infty(\Omega)}$ and $\epsilon$ with optimal dependence on $\epsilon$. We combine these results to bound $\Vert f - \hat f\Vert_{L^\infty(\Omega)}$ with optimal rates in $\epsilon$, which is needed to control the error terms in \cref{theorem:bound-diagonal}.
Since $\Vert f - \hat f\Vert_{L^\infty(\Omega)}\leq \left(2\Vert g\Vert_{L^\infty(\Omega)}+ \Vert g - \hat g\Vert_{L^\infty(\Omega)} \right)\Vert g- \hat g\Vert_{L^\infty(\Omega)}$, we focus on controlling $\Vert g - \hat g \Vert_{L^\infty(\Omega)}$.

The first key argument is to separate $g - \hat g$ as the sum of an approximation error $g - g_{\tau, \epsilon}$ and an estimation error $g_{\tau, \epsilon} - \hat g$ where $g_{\tau, \epsilon} \in \mathcal H_\eta$ and $\Vert g - g_{\tau, \epsilon}\Vert_{L^\infty(\Omega)}\leq C \Vert g \Vert \epsilon^{1-\tilde \nu}$. Using the triangle inequality, controlling $\Vert g - \hat g\Vert_{L^\infty(\Omega)}$ reduces to controlling $\Vert g_{\tau, \epsilon} - \hat g\Vert_{L^\infty(\Omega)}$. We apply the Gargliano-Niremberg inequality with well-chosen parameters to obtain a bound with the product of the Sobolev norm of the estimation error (which is controlled by the $\mathcal H_\eta$ norm) and of the $L^2(\Omega)$ norm, both of which are bounded by the learning approach. Combining all the terms yields the result in \cref{theorem:learning}.

\subsection{Sketch of the proof of \cref{theorem:bound-diagonal}}\label{sec:proof-bound-sketch}
We start by proving a lemma which controls how errors committed at each step accumulate and are eliminated.
\begin{lemma}\label{lemma:bound} Let $T > 1$.
    \begin{align}
        \left\Vert \pi_n^\nu - \hat\pi_n\right\Vert_{TV} \leq \underbrace{C_\mixing\tau_\mixing^n\Vert \hat\pi_0 - \nu\Vert_{TV}}_{A_n}+ \underbrace{\delta_n + C_\mixing\sum_{k=1}^n \tau_\mixing^{n-k-1}\delta_k}_{B_n}
    \end{align}
where $\delta_k = \Vert \barhat{R}_k(\hat \pi_{k-1}) - \bar R_n(\hat \pi_{k-1})\Vert_{TV}$.
\end{lemma}
The gist of \cref{lemma:bound} is that, using the triangle inequality, we can bound $\Vert \pi^\nu_n - \hat\pi_n\Vert$ by $A_n = \Vert\pi^\nu_n - \pi^{\hat \pi_0}_n\Vert$ and $B_n=\Vert\pi^{\hat\pi_0} - \hat \pi_n\Vert$.  \cref{lemma:bound} is proven by combining \cref{prop:lemma_bound} and \cref{prop:optimal-forgetting} in \cref{sec:proof-bound}.

The first source of discrepancy is the intialization error. Indeed, $\hat\pi_k$ is initialized at $\hat \pi_0$ while $\pi_k^\nu$ is initialized at $\nu$ and we isolated this question above by considering the discrepancy between two optimal filters intiialized at $\nu$ and $\hat \pi_0$. The behavior of $A_n$ is known as the stability or forgetting property, and is a property of the optimal filter, i.e. of the dynamical system considered, and not of the algorithms considered. In particuliar, since $R_n$ is mixing for any $n\geq 0$, $A_n$ decreased exponentially as the application of the filter is contractive (under the mixing assumption).

The second source of discrepancy is the accumulation of errors committed at each step by applying the approximate iteration \cref{eq:approx-iteration} in place of \cref{eq:iteration}. A telescopic argument in the proof of \cref{lemma:bound} shows that the accumulation is limited because the forgetting property of the optimal kernel tends to make past errors $\delta_{j}$ for $j < k$ disappear thanks to the $\tau^{n-k-1}$ term. Our argument generalizes the argument in \cite{oudjane} and relies on the projective properties of the Hilbert metric \citep{cohen}. We then combine \cref{lemma:bound} and \cref{theorem:learning} to obtain bound $B_n$ by a constant $\varepsilon$ as small as we want.

\section{Extension: filtering with Generalized Gaussian PSD Models}\label{sec:generalized-psd-models}
Throughout this work, we have focused on Gaussian PSD Models. In fact, many of the properties studied above remain valid for a more general family of models we introduce and study in this section.

Seen as a mixture (with potentially negatively weighted components), Gaussian Mixture Models have components which are aligned with the axes of the space. In cases where $f(u, x) = Q(u, x)$ a transition kernel, we know that $u$ and $x$ are strongly "correlated" (indeed, $Q(u, x)$ is a conditional density) and having non-diagonal precision matrices can be useful. When $Q(u, x)$ is a Gaussian Linear Conditional Distribution, one Gaussian component with non-diagonal precision matrix is enough to approximate $Q$ uniformly. We prove this in \cref{theorem:kalman}.

We introduce Generalized Gaussian PSD Models as the family of non-negative combinations of such components. In this section, we show that Generalized Gaussian PSD Models share many of the properties of Gaussian PSD Models. In addition, we show in \cref{theorem:kalman} that they generalize Kalman filters to more general initial distributions such as multi-modal models.

In this section, for $P$ a positive semi-definite matrix, we denote for any $x, y\in\mathbb R^d$, $k_P(x, y) = e^{-(x-y)^\top  P(x-y)}$ and $C(P) = \int_{\mathbb R^d} k_P(x, 0)dx$.

\begin{definition}[Generalized Gaussian PSD model of order $M$]\label{definition:ggpsd}
A Generalized Gaussian PSD model of order $M$ is a function $f: \mathbb R^d \to\mathbb R$ which can be written:
\begin{equation}\label{eq:def}
    f(x) = Tr(AB(x))
\end{equation}
where $A$ is a positive semi-definite matrix of size $M$ and for any $x\in\mathbb R^d$, $B(x)$ is a positive semi-definite matrix with entries
\begin{align}\label{eq:developed}
    B(x)_{ij}=e^{C_{ij}}k_{P_{ij}}(x, x_{ij})
\end{align}
where $C_{ij} \in \mathbb R$, $x_{ij}\in\mathbb R^d$ and $P_{ij}$ is a $d\times d$ positive semi-definite matrix, for any $1\leq i,j\leq M$.
\end{definition}
We use the notation $f(x; \theta)$ where $\theta = (X, A, P, C)$ and $X = (x_{ij})_{1 \leq i,j\leq M}$, $P=(P_{ij})_{1 \leq i,j\leq M}$, and $C=(C_{ij})_{1 \leq i,j\leq M}$.

Like Gaussian PSD Models, Generalized Gaussian PSD Model are a generalization of Gaussian Mixture Models allowing for negative weights. Indeed, by developing \cref{eq:def} with \cref{eq:developed}, $f$ can be written :  $f(x) = \sum_{i=1}^M\sum_{j=1}^M A_{ij}e^{C_{ij}}k_{P_{ij}}(x, x_{ij})$.

\subsection{Examples of Generalized Gaussian PSD Models}
Below we show that Generalized Gaussian PSD Models generalize Gaussian Mixture Models, Gaussian PSD Models from \cite{ciliberto2021} and indeed, any squared linear combination of Gaussian functions.

\begin{example}[Gaussian Mixture Models are Generalized Gaussian PSD Models]
Let $a\in\mathbb R^d_+$ such that $\sum_{i=1}^d a_i=1$. If $f(x) = \sum_{i=1}^M a_i p(x | \mu_i, P_i)$ with $a_i \geq 0$ and $\sum_{i=1}^Ma_i = 1$, then $f$ verifies \Cref{definition:ggpsd} with $A=\text{diag}(a)$, $P_{ii} = P_i/2$, $x_{ii}=x_i/2$ and $C_{ii} = - \frac{1}{2}\log(C(P_i))$.
\end{example}

\begin{example}[Gaussian PSD Models are Generalized Gaussian PSD Models]
Consider a Gaussian PSD Model $f(x) = \Phi_M(x)^\top  A \Phi_M(x)$ with $\Phi_M(x) = (k_\eta(x, x_1), \ldots, k_\eta(x, x_M)^\top\in\mathbb R^M$. Then, $f$ a Generalized Gaussian PSD Model. Indeed,
$f(x) = Tr(A \Phi_M(x) \Phi_M(x)^\top )$.
Each component of the positive semi-definite matrix $\Phi(x)\Phi(x)^\top $ is a function of the form $x \mapsto e^{C_{ij}}k_{2\eta}(x, \frac{x_i + x_j}{\sqrt{2}})$ with $C_{ij}=\Vert x_i\Vert^2 + \Vert x_j \Vert^2 -\frac{\Vert x_i + x_j\Vert^2}{2}$, and $f$ verifies \cref{definition:ggpsd}.

\end{example}
\begin{example}[Squared Gaussian Linear Models are Generalized Gaussian PSD Models]
Consider the function $f(x) = \left( w^\top \Phi_M(x)\right)^2$ where $\Phi_M(x) = (k_{P_1}(x, x_1), \ldots, k_{P_M}(x, x_M)^\top $ and $w\in\mathbb R^d$. Then, $f$ is a Generalized Gaussian PSD Model. Indeed, $f(x) = Tr\left(\Phi_M(x)^\top ww^\top \Phi_M(x)\right)=Tr\left(ww^\top  \Phi_M(x)\Phi_M(x)^\top \right)$. Since $ww^\top $ and $\Phi_M(x)\Phi_M(x)^\top $ are both positive semi-definite matrices, $f(x)\geq 0$ and $\Phi_M(x)\Phi_M(x)^\top _{ij}=e^C_{ij}k_{P_{ij}}(x, x_{ij})$ with $P_{ij} = P_i + P_j$, $x_{ij} = P_{ij}^{-1/2}\left(P_ix_i + P_jx_j\right)$.
\end{example}

\subsection{Properties of Generalized Gaussian PSD Models}

\subsubsection{Closed-form stability with respect to probabilistic operations}
Like Gaussian PSD Models, Generalized Gaussian PSD Models are closed under product, partial evaluation, and marginalization.
\begin{proposition}\label{prop:ops}
    Let $f(x, y; \theta_1)$ and $g(y, z; \theta_2)$ be two Generalized Gaussian PSD Models of order $M_1$ and $M_2$ respectively, where all precision matrices are positive definite.
    \begin{itemize}
        \item\textbf{Integral over $\mathbb R^d$} $\int f(x, y; \theta_1)dxdy$ is given by the algorithm $\integralpsd(\theta_1)$.
        \item\textbf{Partial evaluation}$f(x, y_0; \theta_1) = h(x; \theta^\prime)$ is a Generalized Gaussian PSD Model of order $M_1$ and $\theta^\prime$ is given by the algorithm $\partialpsd(y_0, \theta_1)$.
        \item\textbf{Product} $f(x, y; \theta_1)g(y, z; \theta_2)=h(x, y, z; \theta^\prime)$ is a Generalized Gaussian PSD Model of order $M_1\times M_2$ and $\theta^\prime$ is given by the algorithm $\productpsd(\theta_1, \theta_2)$.
        \item\textbf{Marginalization} $\int f(x, y; \theta_1)dy=h(x; \theta^\prime)$ is a Generalized Gaussian PSD Model of order $M_1$ and $\theta^\prime$ is given by the algorithm $\marginalpsd(y, \theta_1)$.
    \end{itemize}
    \end{proposition}
\paragraph{Proof sketch for the product} Using \cref{eq:def}, $f(x, y)g(y, z)$ can be written $h(x, y, z) = Tr(A_1\otimes A_2 \times B_1(x, y) \otimes B_2(y, z))$. The entries of $B_1(x, y) \otimes B_2(y, z)$ are products of Gaussian functions which can simplified into the the form \cref{eq:developed}. The complete proof and description of each operation can be found in \cref{sec:proof-ops}.

\subsubsection{Closed-form filtering iteration}
An optimal filtering iteration \cref{eq:iteration} can be written using the four operations of \cref{prop:ops}.
\begin{proposition}\label{proposition:psdfilterstep}
Let $\mu(x; \theta_\mu)$, $q(x, x^\prime; \theta_q)$ and $g(x, y; \theta_g)$ be three Generalized Gaussian PSD models with order $M$, $M_q$ and $M_g$ respectively. Let $y\in\mathbb R^d$ such that $\int\int q(u, x)g(x, y)\mu(u)dudx>0$. The density $\mu^\prime$ defined by $\mu^\prime(x) = \frac{\int q(u, x)g(u, y)\mu(u)du}{\int\int q(u, x)g(x, y)\mu(u)dudx}$ is a Generalized Gaussian PSD Model with order at most $M \times M_q \times M_g$ whose parameters are given by $\filtersteppsd(y, \theta_\mu, \theta_q, \theta_g)$.
\end{proposition}

Applying \cref{proposition:psdfilterstep} recursively as in \cref{alg:fullfilter} to compute an approximate filter $\hat \pi_n$, the order of $\hat\pi_n$ increases exponentially with $n$. A constant number of anchor points can be used by compression $\hat\pi_n$ at each step, for example by learning a Gaussian PSD Model with a given number of anchor points (indeed $\hat\pi_n$ is a smooth sum-of-squares), which is justified by \cref{theorem:learning}.

\subsubsection{Generalized Gaussian PSD Models generalize Kalman filters} 

Conditional Gaussian Linear Distributions are widely used in filtering and dynamical modeling since they cover transition or observation state-space equations such as $X_{t+1} = FX_t + b + \Sigma U_t$ where $U_t$ is Gaussian noise considered in the Kalman filter and extensions. 

\begin{theorem}[Approximating a Conditional Gaussian Linear Distribution]\label{theorem:kalman}
Let $p$ be a Conditional Gaussian Linear Distribution defined by $p(y | x) = \mathcal N(y | Fx + b, \Sigma)$ with $F\in\mathbb R^{d ^\prime\times d}$, $b\in\mathbb R^{d^\prime}$ and $\Sigma\in\mathcal S^{++}_{d^\prime}(\mathbb R)$. Then for any $\epsilon>0$ and $R>0$, there exists a Generalized Gaussian PSD Model of order $1$ such that $\vert p(y | x) - \hat p(x, y) \vert \leq \epsilon, $$\forall x, y \in \mathbb R^d \times \mathbb R^{d^\prime}$ such that $\Vert x\Vert_2^2 + \Vert y \Vert_2^2 \leq R^2$.
\end{theorem} 

This shows that a Generalized Gaussian PSD Model of order $1$ can approximate a Conditional Gaussian Linear Model with arbitrary accuracy on a compact. Note that in this case, applying \cref{alg:fullfilter} with $\hat Q$ and $\hat G$ such approximations (each of order $1$) and $\hat\pi_0$ of order $M$ yields an approximation of $\pi_n$ of constant order $M$. The proof of \cref{theorem:kalman} can be found in \cref{sec:proof_kalman}. 

Generalized Gaussian PSD Models can be used to learn general transition kernels using non-convex optimization. We discuss this in \cref{sec:learning-general}.

\newpage

\section*{Acknowledgments}
T.C. gratefully acknowledges support from the French National Agency for Research, grant ANR-18-CE40-0016-01.
C.C. acknowledges the support of the Royal Society (grant SPREM RGS\textbackslash{}R1\textbackslash{}201149) and Amazon.com Inc. (Amazon Research Award – ARA).
B.G. acknowledges partial support by the U.S. Army Research Laboratory and the U.S. Army Research Office, and by the U.K. Ministry of Defence and the U.K. Engineering and Physical Sciences Research Council (EPSRC) under grant number EP/R013616/1; B.G. also acknowledges partial support from the French National Agency for Research, grants ANR-18-CE40-0016-01 and ANR-18-CE23-0015-02. 
A.R. acknowledges partial support from the French government under management of Agence Nationale de la Recherche as part of the “Investissements d’avenir” program, reference ANR-19-P3IA-0001 (PRAIRIE 3IA Institute), and support from the European Research Council (grant REAL 947908).
\bibliography{biblio}
\appendix
\newpage

\section{Markov kernels \& Hidden Markov Models}

\subsection{Tools and notation}
Let $(E, \mathcal E)$ be a measurable space. We denote $\mathcal M(E)$ the set of finite signed measures on $(E, \mathcal E)$, $\mathcal M_+(E)$ the set of finite positive measures, $\mathcal M_0(E)$ the set of finite signed measures which sum to $0$ and $\mathcal P(E)$ the set of probability distributions on $(E, \mathcal E)$. Let $(F, \mathcal F)$ be a second measurable space.

\subsection{Total variation}

\begin{definition}[Total variation norm]
Let $\xi$ be a finite signed measure on $(E, \mathcal E)$. The total variation of $\xi$ is $\Vert \xi \Vert_{TV}= \xi^+(E) + \xi^-(X)$ where $\xi^+, \xi^-$ is the Jordan-Hahn decomposition of $\xi$. If $\xi$ admits a density with respect to the Lebesgue measure, then $\Vert \xi \Vert_{TV} = \int \vert \xi(x)\vert dx$.
\end{definition}
\begin{proposition} Let $\mu$ and $\nu$ be two finite measures on $(E, \mathcal E)$. Then,
\begin{align}
    \Vert \bar \mu - \bar \nu \Vert = \frac{\Vert \mu - \nu \Vert}{\mu(E)} + \frac{\vert \mu(E) - \nu(E)\vert}{\mu(E)}
\end{align}
In particular,
\begin{align}
    \Vert \bar \mu - \bar \nu \Vert \leq \frac{2\Vert \mu - \nu \Vert}{\mu(E)}
\end{align}
\end{proposition}
\subsection{Transition kernels}
We report the essential results relative to transition kernels taken \cite{cappehmm}.
\begin{definition}[Transition kernel]\label{def:transition-kernel}
A function $Q:E \times \mathcal F \to \mathbb R^+$ is an unnormalized transition kernel if:
\begin{itemize}
    \item for all $x\in E$, $Q(x, \cdot)$ is a positive measure on $(F, \mathcal F)$;
    \item for all $A\in \mathcal F$, $x\mapsto Q(x, A)$ is measurable.
\end{itemize}
$Q$ is normalized if for any $x\in E$, $Q(x, F)=1$. When $E=F$ and $Q$ is normalized, $Q$ is said to be a Markov transition kernel.

By abuse of notation, when $Q$ admits a density with respect to the Lebesgue measure, we denote it $Q$ as well, i.e. $Q(x, dy)=Q(x, y)dy$.
\end{definition}

Note that $R_n(u, x)= Q(u, x)G(x, z_n)$ is indeed an unnormalized transition kernel.

\begin{definition}[Effects of kernels]
    Let $K$ be an unnformalized kernel on $(E, \mathcal E)\times (F, \mathcal F)$, $\mu\in\mathcal M_+(E)$ and $f$ a bounded function $f:E\to\mathbb R$.

    Then, $K\mu\in\mathcal M_+(F)$ with for any $A\in\mathcal F$,
    \begin{align}
        K\mu(A) = \int K(u, A)\mu(du),
    \end{align}
    
    and $Kf: E \to \mathbb R$ is a bounded function with for any $u\in E$,
    \begin{align}
        Kf(u) = \int K(u, A)\mu(du).
    \end{align}
\end{definition}
\subsection{Hidden Markov Models}
\begin{definition}[Hidden Markov Model]
Let $(E, \mathcal E)$ and $(F, \mathcal F)$ be two measurable spaces. Let $Q$ and $G$ denote a Markov transition kernel on $(E, \mathcal E)$ and $G$ denote a transition kernel from $(E, \mathcal E)$ to $(F, \mathcal F)$. Let $T$ be the Markov transition kernel defined on the product space $(E \times F, \mathcal E \otimes \mathcal F)$ by 
\begin{equation}
    \forall (x, y) \in E\times F, \forall C\in\mathcal E \otimes\mathcal F, ~T\left[(x, y), C\right] = \int \int 1_C((x^\prime, y^\prime)) Q(x, dx^\prime) G(x^\prime, dy^\prime)
\end{equation}

The Markov Chain $\lbrace X_k, Y_k \rbrace_{k\geq 0}$ with Markov transition kernel $T$ and initial distribution $\nu \otimes G$, where $\nu$ is a probability distribution on $(E, \mathcal E)$ is called a \textbf{Hidden Markov Model}.

We denote $\mathbb P_\nu$ and $\mathbb E_\nu$ the probability measure and corresponding expectation associated with the process $\lbrace (X_k, Y_k)\rbrace$ over $\left((E\times F)^{\mathbb N}, (\mathcal E \otimes \mathcal F)^{\otimes \mathbb N}\right)$.
\end{definition}

Throughout this work, we assume that for any $x\in E$, $Q(x, \cdot) \ll \lambda(\cdot)$ and $G(x, \cdot) \ll \lambda(\cdot)$ where $\lambda$ is the Lebesgue measure over $(E, \mathcal E)$, and we denote $Q(x, dx^\prime) = Q(x, x^\prime)dx^\prime$ and $G(x, dy)= G(x, y)dy$.

\begin{definition}[Filtering distribution]
Let $\nu$ be a probability distribution over $(E, \mathcal E)$ and $n \geq 0$. We denote $\pi_n^\nu$ the conditional distribution of $X_n$ given $Y_{1:n}$, i.e.
\begin{itemize}
    \item $\pi_n^\nu$ is a transition kernel from $F^n$ to $E$
    \item $\pi_n^\nu$ satisfies for any bounded function $f: E \to \mathbb R$,
    \begin{equation}
        \mathbb E_\nu \left[ f(X_n) \vert Y_{1:n}\right] = \int f(x)\pi^\nu_n(Y_{1:n}, dx)
    \end{equation}
\end{itemize}
\end{definition}

\subsection{Mixing kernels}
\begin{definition}[Mixing kernel]
    We say a kernel $K(x, dy)$ is mixing if there exists a positive constant $\mixing > 0$ and a non-negative measure $\xi$ such that for any $x\in E$ and $A\in\mathcal F$,
    \begin{equation}
        \mixing\xi(A) \leq K(x, A) \leq \frac{1}{\mixing}\xi(A).
    \end{equation}
\end{definition}
If $K$ is mixing for a measure $\xi$ and a constant $\mixing$ we write that $K$ is $\mixing$-$\xi$-mixing.

\begin{remark}
    Note that we can add the constraint that $\xi$ be normalized. Indeed, if $K$ is $\mixing$-$\xi$-mixing with $\xi(E)\neq 1$, then $K$ if $\bar \mixing$-$\bar\xi$-mixing with $\bar\mixing = \mixing \times \min\left(\xi(E), \frac{1}{\xi(E)}\right)\leq \mixing$.
\end{remark}

\begin{proposition}[Sufficient condition for mixing when $K$ admits a density]
    If $K(x, dy) = \kappa(x, y)dy$ and there exists $\mixing >0$ and a measure density $\xi$ such that for any $x\in E$ and $y\in F$,
    \begin{equation}
        \mixing\xi(y) \leq \kappa(x, y) \leq \frac{1}{\mixing}\xi(y).
    \end{equation}

    then $K$ is mixing with constant $\mixing$ and $\xi(A)=\int 1_A(y)\xi(y)dy$.
\end{proposition}

\begin{proposition}
    If $K$ is mixing with $\mixing$ and $\xi$, then for any $\mu$,
    \begin{equation}
        \mixing\xi(A)\leq K\mu(A)\leq \frac{1}{\mixing}\xi(A).
    \end{equation}
\end{proposition}
\begin{proof}
    \begin{equation}
    \mixing\xi(A)\leq K\mu(A)=\int 1_A(x)\int K(u, dx)\mu(du) \leq \int 1_A(x)\int \frac{1}{\mixing}\xi(dx)\mu(du)=\frac{1}{\mixing}\xi(A)
    \end{equation}
\end{proof}

\subsection{Optimal kernel $R_n$}\label{sec:optimal_kernel_appendix}

\begin{definition}
Let $y_n \in F$.
Let $R_n : E \times \mathcal E \to \mathbb R^+$ a kernel defined by:
    for any bounded function $f\in\mathcal F_b(E)$, and $u\in E$, 
    \begin{align}
        R_n(u, f) = \int f(x)Q(u, x)G(x, y_n)dx
    \end{align}
We call $R_n$ the optimal kernel.
\end{definition}
For alternative definitions see for example \citealp[page 220]{cappehmm}.

\subsection{Hilbert metric}

\begin{definition}[Comparable measures] Let $\mu$ and $\nu$ be two measures on $(E, \mathcal E)$. $\mu$ and $\nu$ are said to be comparable if there exists $\infty > a, b >0$ such that for any $A\in \mathcal E$,
\begin{equation}
    a\nu(A) \leq \mu(A) \leq b\nu(A)
\end{equation}
\end{definition}

\begin{proposition}
Let $\mu$ and $\nu$ be two comparable measures on $(E, \mathcal E)$ and let $K$ be an unnormalized transition kernel on $(E, \mathcal E)$. Then, for any $n\geq 0$, $K^n\mu$ and $K^n\nu$ are comparable.
\end{proposition}

\begin{proof} By recursion,
    \begin{equation}
        bK\nu(A) \leq K\mu(A) = \int 1_A(x) \int K(u, dx)\mu(du) \leq aK\nu(A)
    \end{equation}
\end{proof}

\begin{definition}[Hilbert metric]
   Let $\mu$ and $\nu$ be two comparable probability distributions on $(E, \mathcal E)$. The Hilbert metric between $\mu$ and $\nu$ is defined as 
\begin{align}
    h(\mu, \nu) = \log\left[\frac{\sup_{A\in\mathcal E, \nu(A)>0} \frac{\mu(A)}{\nu(A)}}{\inf_{A\in \mathcal E, \nu(A) > 0}\frac{\mu(A)}{\nu(A)}}\right]
\end{align}
\end{definition}

\begin{proposition} If $\mu$ and $\nu$ are two comparable probability distributions on $(E, \mathcal E)$ and furthermore they both admit densities then,
\begin{equation}
h(\mu, \nu)=\log\left[\frac{\textrm{ess}\sup_x \frac{\mu(x)}{\nu(x)}}{\textrm{ess}\inf_x\frac{\mu(x)}{\nu(x)}}\right] = \log \left[\left\Vert\frac{\mu(x)}{\nu(x)}\right\Vert_\infty\left\Vert\frac{\nu(x)}{\mu(x)}\right\Vert_\infty\right]
\end{equation}
\end{proposition}

\begin{definition}[Birkhoff contraction coefficient]
Let $K$ be an unnormalized transition kernel. Define $\tau(K)$ such that 
\begin{equation}
    \tau(K) = \sup_{0 < h(\mu, \nu) < \infty} \frac{h(K\mu, K\nu)}{h(\mu, \nu)}
\end{equation}
where the supremum is taken over comparable, positive measures $\mu$ and $\nu$.
\end{definition}

\begin{proposition}[Properties of the Birkhoff coefficient]\label{prop:hilbert}
Let $K$ be an unnormalized transition kernel.
\begin{itemize}
    \item $\tau$ is sub-multiplicative, i.e. $\tau(KL)\leq \tau(K)\tau(L)$
    \item $\tau \leq 1$
    \item if in addition $K$ is $\mixing$-$\xi$-mixing, then $\tau(K) \leq \frac{1- \mixing^2}{1 + \mixing^2}$
\end{itemize}
\end{proposition}

These properties are proven in \cite{cohen}, which studies the Hilbert metric is detail.

\begin{proof} We have $\tau(K) = \tanh\left[\frac{1}{4}\Delta(K)\right]$ where $\Delta(K) = \sup_{\mu, \mu^\prime}h(K\mu, K\mu^\prime)$ ($\Delta$ for diameter). Since $K$ is $\mixing$-$\xi$-mixing, if $\mu$ and $\nu$ are two finite measures on $(E, \mathcal E)$,
\begin{equation}
    \mixing^2K\mu(A)\leq\mixing\xi(A)\leq K\mu(A)\leq \frac{1}{\mixing}\xi(A) \leq \frac{1}{\mixing^2}K\mu^\prime(A)
\end{equation}
   So, $h(K) \leq \log\left(\frac{1}{\mixing^4}\right)\leq \frac{1 - \mixing^2}{1 + \mixing^2}$.
\end{proof}

\begin{proposition}[Total variation - Hilbert comparaisons]\label{prop:comparaison}
Without any hypotheses on $\mu, \nu$ finite measures, if $\bar \mu = \mu / \mu(E)$ and $\bar \nu = \nu / \nu(E)$ are normalized counterparts to $\mu$ and $\nu$, then
    \begin{equation}\label{eq:tv-to-hilbert}
        \Vert \bar\mu - \bar\nu \Vert \leq \frac{2}{\log(3)}h(\mu, \nu)
    \end{equation}
If in addition, $K$ is an $\mixing$-mixing kernel,
\begin{equation}\label{eq:hilbert-to-tv}
    h(K\mu, K\nu) \leq \frac{1}{\mixing^2} \Vert \mu - \nu \Vert
\end{equation}
These inequalities are proven in \cite{cohen}.

\end{proposition}
\section{Proof of \cref{theorem:learning}}\label{sec:proof-learning}
\subparagraph{Notation and background results} In this section, $\Vert \cdot \Vert$ without any subscript denotes $\Vert \cdot\Vert_{W_2^\beta(\Omega)}$. Let $f$ a target function defined on $\Omega = (-1, 1)^d$ that verifies \cref{assumption:target_function}, and let $g = \sqrt{f}$. Let $\nu = \min(1, d/2\beta)$. Let $\epsilon > 0$ and $\delta > 0$. In this proof all constants $C$ and exponents $\alpha, \gamma, \rho, \ldots$ are independent of $f, g, \hat g, \hat f$ unless otherwise stated. Only $\beta, \nu$, $\tilde nu$ and $\theta$ have importance. We recall also the Gagliardo-Nirenberg inequality, that will be used later.
\begin{lemma}[Gagliardo-Nirenberg inequality\cite{wendland2004scattered}]\label{lemma:gargliano}
Let $u\in L^\infty(\mathbb R^d)\cap W^{m, 2}(\mathbb R^d)$ then
    \begin{align}
        \Vert u \Vert_{L^\infty(\Omega)} \leq C \Vert u\Vert_{W^{m, 2}(\Omega)}^\theta \Vert u\Vert_{L^2(\Omega)}^{1-\theta}
    \end{align}
    where $\theta = \frac{d}{m}$ and $C$ is independent of $u$.
\end{lemma}

\subparagraph{Setting the function space $\mathcal H_\eta$ and $g_{\tau, \epsilon}$}
Set $\tau = \epsilon^{-2/\beta}$ and $\lambda = \epsilon^{\frac{2\beta + d}{\beta}}$ and $\mathcal H_\eta$ the reproducing kernel Hilbert space associated to $k_\eta$ where $\eta = \tau 1_d$.

\subparagraph{Existence and properties of $g_{\tau, \epsilon}$}
As a consequence of the Stein extension theorem (see Corollary A.3 of \cite{ciliberto2021}), there exists a function $\tilde g \in W^\beta_2(\Omega)$ such that $\tilde g_{\vert \Omega}=g$ and $\Vert \tilde g\Vert_{W_2^\beta(\Omega)} \leq\Vert g\Vert_{W_2^\beta(\Omega)}$ and $\Vert \tilde g\Vert_{L^\infty(\Omega)}\leq C\Vert g\Vert_{L^\infty(\Omega)}$.
According to \cite{ciliberto2021} and \cite{sampling-ulysse} (Proposition 7) there exists $g_{\tau, \epsilon}\in\mathcal H_\eta$ and some constants $C_1, C_2$ depending only on $\beta, d$ and independent of $g$ such that:
\begin{align}
&\Vert g - g_\tau \Vert_{L^\infty(\Omega)} \leq C_1 \epsilon^{1- \nu}\Vert g\Vert_{W^\beta_2(\Omega)}\label{eq:g-g-tau},\\
&\Vert g_\tau \Vert_{\mathcal H_\eta}\leq C_2 \Vert g\Vert_{W^\beta_2(\Omega)}\epsilon^{-d/2\beta}\label{eq:g-tau}.
\end{align}
 
\subparagraph{Learning $\hat g$ and $\hat f$} Let $M, n \in \N$, draw $X\in\mathbb R^{n \times d}$ the training set and $\tilde X \in \mathbb R^{M \times d}$ the set of anchor points. Define $y \in \mathbb R^n$ such that $y_i = g(x_i)$ for any $1 \leq i \leq n$ and $x_i$ is the $i$-th row of $X$.
We formalize \cref{eq:learning-problem} explicitly as a kernel ridge regression problem below:
\begin{align}\label{eq:learning-problem-appendix}
    \min_{a\in\mathbb R^d} \frac{1}{n} \vert a^T\Phi_{\eta}(X_i) - y_i\vert^2 + \lambda a^TK a
\end{align}
where $\Phi_\eta(x) = (k_\eta(x, \tilde x_1) \dots k_\eta(x, \tilde x_M))^T\in\mathbb R^M$ and $K\in\mathbb R^{M\times M}$ is given by $K_{ij}=k_\eta(\tilde x_i, \tilde x_j)$.
Problem \cref{eq:learning-problem-appendix} is strongly convex and has a unique solution $\hat a$.
We denote $\hat g$ the Gaussian Linear Model defined by $\hat a, \eta,$ and $\tilde X$, i.e. such that for any $x \in\mathbb R^d$, $\hat g(x) = \hat a^T\Phi_\eta(x)$. Define $\hat f$ the Gaussian PSD Model $\hat f= \hat g^2$ where $\hat f(x) = \Phi_\eta(x)^T\hat a\hat a^T\Phi_\eta(x)$.
The analysis of \cite{sampling-ulysse} shows that when $M \geq C' \log^d(\frac{1}{\epsilon})\log(\frac{1}{\delta\epsilon})\epsilon^{-d/\beta}$, and $n \geq C' \epsilon^{-d/\beta}\log \frac{1}{\delta}$, then there exist two constants $C_3$ and $C_4$ independent of $g$ and $\hat g$ such that $\hat g$ verifies the following inequalities each with probability at least $1-\delta$,
\begin{align}
   &\Vert \hat g\Vert_\mathcal H\leq C_3\Vert g\Vert_{W^\beta_2(\Omega)}\epsilon^{-d/2\beta}\label{eq:hat-g}\\ 
    &\Vert g - \hat g\Vert_{L^2(\Omega)} \leq C_4 \Vert g\Vert_{W^\beta_2(\Omega)}\epsilon.\label{eq:g-hat-g}
\end{align}

\subparagraph{Deriving the bound for $g - \hat{g}$ in $L^\infty(\Omega)$}
In particular, using the triangle inequality and combining \cref{eq:g-g-tau} and \cref{eq:g-hat-g}, there exists a constant $C_5$ such that with probability at least $1-\delta$, 
\begin{align}
    &\Vert g_{\tau, \epsilon} - \hat g\Vert_{L^2(\Omega)} \leq C_5\Vert g\Vert_{W^\beta_2(\Omega)}\epsilon.\label{eq:g-tau-hat-g}
\end{align}
We now have all the ingredients to bound $\Vert g - \hat g\Vert_{L^\infty(\Omega)}$ in high probability. First, notice that :
\begin{align}
   \Vert g - \hat g\Vert_{L^\infty(\Omega)} \leq \Vert g - g_{\tau, \epsilon}\Vert_{L^\infty(\Omega)}+ \Vert g_{\tau, \epsilon} - \hat g\Vert_{L^\infty(\Omega)}.
\end{align}
We apply the Gagliardo-Nirenberg inequality (\cref{lemma:gargliano}) to $\Vert g_{\tau, \epsilon} - \hat g\Vert_{L^\infty(\Omega)}$:
\begin{align}
\Vert \hat g - g_{\tau, \epsilon}\Vert_{L^\infty(\Omega)}\leq C \Vert \hat g - g_{\tau, \epsilon}\Vert_{W^{m}_2(\Omega)}^\theta \Vert \hat g - g_{\tau, \epsilon}\Vert_{L^2(\Omega)}^{1-\theta}.
\end{align}
with $\theta = \frac{d}{2m}$ and $m \geq d/2$ (we fix $m$ when we optimize the exponents below) and $C$ is a constant independent of $\hat g$ and $g_{\tau, \epsilon}$.

The Sobolev norm $\Vert h \Vert_{W^m_2}(\mathbb R^d)$ is upper bounded by the rkhs norm $\Vert h \Vert_{\mathcal H_\eta}$ for $h\in \mathcal H_\eta \subset W^m_2(\mathbb R^d)$. Thus, applying the triangle inequality and bounds \cref{eq:g-tau,eq:hat-g,eq:g-tau-hat-g} there exists a constant $C_6$  such that with probability at least $1 - 2\delta$,
\begin{equation}\label{eq:g-hat-g-tau-infty}
\Vert \hat g - g_{\tau, \epsilon}\Vert_{L^\infty(\mathbb R^d)}\leq C_6\Vert g\Vert\epsilon^{1 - \theta - \theta d/2\beta}.
\end{equation}

Combining \cref{eq:g-g-tau,eq:g-hat-g-tau-infty}, there exist two constants $C, C^\prime > 0$ such that with probability at least $1 - 2\delta$,
\begin{align}
   \Vert g- \hat g\Vert_{L^\infty(\Omega)} \leq C\Vert g\Vert_{W^\beta_2(\Omega)} \epsilon^{1-\nu} + C^\prime\Vert g\Vert_{W^\beta_2(\Omega)}\epsilon^{1- \theta - \theta d/2\beta} 
\end{align}

Choosing $m = \beta + \frac{d}{2}$, there exists a constant $C_7>0$  such that for $\epsilon$ small enough with probability at least $1 - 2\delta$, 
\begin{align}
    \Vert g - \hat g\Vert_{L^{\infty}(\Omega)}\leq C\Vert g\Vert_{W^\beta_2(\Omega)}\epsilon^{1- d/2\beta}.
\end{align}

\subparagraph{Bounding $\Vert g+ \hat g\Vert_{L^\infty(\Omega)}$}
Using the triangle inequality and \cref{eq:hat-g}, there exists a constant $C_8>0$ and $\rho > 0$ such that with probability at least $1- \delta$,
\begin{align}
    \Vert g +\hat g\Vert_{L^\infty(\Omega)} \leq 2\Vert g\Vert_{L^\infty(\Omega)} + \Vert g-\hat g\Vert_{L^\infty(\Omega)} \leq 2\Vert g\Vert_{L^\infty(\Omega)} + C\Vert g\Vert_{W^\beta_2(\Omega)}\epsilon^{1- d/2\beta}.
\end{align}

\subparagraph{Bounding $\Vert f - \hat f\Vert_\infty$}
Notice that since $a^2 - b^2 = (a-b)(a+b)$,
\begin{align}
    \Vert \hat f - f \Vert_{L^\infty(\Omega)} \leq \Vert g - \hat g\Vert_{L^\infty(\Omega)} \Vert g +\hat g\Vert_{L^\infty(\Omega)}.
\end{align}

Combining the above bounds:

\begin{align}
    \Vert \hat f - f \Vert_{L^\infty(\Omega)} \leq C\Vert g\Vert^2_{W^\beta_2(\Omega)}\epsilon^{2- d/\beta} + C^\prime \Vert g\Vert^2_{W^\beta_2(\Omega)}\epsilon^{1 - d/2\beta}
\end{align}

Since $\beta > d/2$, under the conditions on $M, n$, there exists a constant $C''= C + C'$ depending only on $\Omega, d, \beta$ and independent of $f, g,\hat g, \hat f$ such that with probability at least $1-3\delta$,

\begin{align}
    \Vert \hat f - f \Vert_{L^\infty(\Omega)} \leq C''\Vert g\Vert^2_{W^\beta_2(\Omega)} \epsilon^{1 - d/2\beta}
\end{align}

\section{Proof of \cref{theorem:bound-diagonal}}\label{sec:proof-bound}
In this section $E = (-1, 1)^d$ and $\Omega = E \times F$. Without loss of generality, we assume that $F=E$. Of course, any compact can be considered.

\subsection{Propagation of one-step errors}

We generalize the proof technique in \cite{oudjane} to take into consideration general sequences of densities. 

\begin{proposition}
Let $\pi_0\in\mathcal P(E)$. Let $(z_k)_{k\geq 1}$ and $(\pi_k)_{k\geq 1}$ the optimal filter sequence computed on the $(z_k)_{k\geq 1}$ and initialized at $\pi_0$. Let $(\mu_k)\in\mathcal P(E)^\mathbb N$ a sequence of distributions such that $\mu_0 = \pi_0$. Then, for any $n\geq 0$:
\begin{align}\label{eq:telescopic}
    \mu_n - \pi_n = \sum_{k=1}^n \bar R_{n:k+1}(\mu_k) - \bar R_{n:k}(\mu_{k-1})
\end{align}
with the notation that $\bar R_{n:k} = \bar R_n \circ \dots \circ \bar R_k$ and $R_{n+1:n}=id$ if $k > l$.
\end{proposition}
\begin{proof}
    Telescopic sum:
\begin{align}
    \mu_n - \pi_n &= \sum_{k=1}^n \bar R_{n:k+1}(\mu_k) - \bar R_{n:k}(\mu_{k-1})\\
    &= \mu_n - \bar R_n(\mu_{n-1}) + \bar R_{n}(\mu_{n-1} - \bar R_{n:n-1}(\mu_{n-2}) \ldots + \bar R_{n:2}(\mu_1) - \underbrace{\bar R_{n:1}(\mu_0)}_{=\bar R_{n:1}(\pi_0)=\pi_n}
\end{align}
\end{proof}

\begin{proposition}[Optimal filter stability]\label{prop:optimal-forgetting}
Let $(\pi_n^\nu)$ and $(\pi_n^\mu)$ two sequences of optimal filters initialized at $\nu$ and $\mu$ respectively, and computed on the same data sequence $z_1, \ldots, z_n$. We assume that for all $n\geq 1$, $R_n$ verifies \cref{assumption:mixing}. Then,
\begin{align}
    \Vert \pi^\nu_n - \pi^\mu_n \Vert \leq \frac{2}{\mixing^2\log 3}\left(\frac{1 - \mixing^2}{1 + \mixing^2}\right)^{n-1}\Vert \mu - \nu \Vert_{TV}
\end{align}
\end{proposition}
 
\begin{proof}
   By \cref{eq:tv-to-hilbert},
\begin{align}
    \Vert \pi^\mu_n - \pi^\nu_n \Vert_{TV} & \leq h(R_{n:1}(\mu_n), R_{n:1}(\nu_n))
\end{align}
Since $R_{n}, \ldots, R_1$ are mixing, $R_{1}\mu$ and $R_1\nu$ are comparable and we can apply \cref{prop:hilbert} with the Hilbert contraction coefficient $\tau(R_{n:2})\leq \tau_\mixing^{n-1}\leq\left(\frac{1-\mixing^2}{1+\mixing^2}\right)^{n-1}$.
\begin{align}
    \Vert \pi^\mu_n - \pi^\nu_n \Vert_{TV} & \leq \frac{2}{\log 3}\tau_\mixing^{n-1}h(R_{1}(\mu), R_{1}(\nu))
\end{align}
And finally, applying \cref{eq:hilbert-to-tv},
\begin{align}
    \Vert \pi^\mu_n - \pi^\nu_n \Vert_{TV} & \leq \frac{2}{\mixing^2\log(3)}\tau_\mixing^{n-1}\Vert \mu - \nu\Vert_{TV}
\end{align}
\end{proof}
\begin{proposition}\label{prop:lemma_bound}
    Let $E = (-1, 1)^d$. Let $\mixing>0$ and $\xi\in\mathcal P(E)$ such that the optimal kernel $R_n:E\times E\to \mathbb R^+$ is $\mixing$-$\xi$ mixing for all $n$. Let $\mu_n$ a sequence of positive, finite measures on $E$ such that $\mu_0=\pi_0$. Then,
\begin{align}
    \Vert \mu_n - \pi_n \Vert_{TV} & \leq \delta_n + \frac{2}{\log 3}\frac{1}{\mixing^2}\sum_{k=1}^{n-1}\tau^{n-k-1}\delta_k
\end{align}
where $\delta_n = \Vert \mu_n - \bar R_n(\mu_{n-1})\Vert_{TV}$.
\end{proposition}

\begin{proof}
Applying the triangle inequality with the total variation norm to \cref{eq:telescopic},
\begin{align}
    \Vert \mu_n - \pi_n \Vert_{TV} &\leq \sum_{k=1}^n \Vert \bar R_{n:k+1}(\mu_k) - \bar R_{n:k}(\mu_{k-1}) \Vert_{TV} & \\
    & \leq \Vert \mu_n - \bar R_n(\mu_{n-1})\Vert_{TV} +\sum_{k=1}^{n-1} \Vert \bar R_{n:k+1}(\mu_k) - \bar R_{n:k}(\mu_{k-1}) \Vert_{TV} &
\end{align}
We apply \cref{eq:tv-to-hilbert} from \cref{prop:comparaison}, 
\begin{align}
    \Vert \mu_n - \pi_n \Vert_{TV}& \leq \Vert \mu_n - \bar R_n(\mu_{n-1})\Vert_{TV} + \frac{2}{\log 3}\sum_{k=1}^{n-1}h(R_{n:k+1}(\mu_k), R_{n:k}(\mu_{k-1}))
\end{align}
Since $R_{k+1}$ and $R_k$ are mixing, $R_{k+1}\mu_k$ and $R_{k+1}R_{k}\mu_{k-1}$ are comparable and we can apply \cref{prop:hilbert} with the Hilbert contraction coefficient $\tau(R_{n:k+2})\leq \tau_\mixing^{n-k-1}= \left(\frac{1-\mixing^2}{1+\mixing^2}\right)^{n-k-1}$.
\begin{align}
\Vert \mu_n - \pi_n \Vert_{TV}& \leq \Vert \mu_n - \bar R_n(\mu_{n-1})\Vert + \frac{2}{\log 3}\sum_{k=1}^{n-1}\tau_\mixing^{n-k-1}h(R_{k+1}\mu_k, R_{k+1}R_k\mu_{k-1})
\end{align}
Since $R_{k+1}$ is mixing, we apply \cref{eq:hilbert-to-tv} from \cref{prop:comparaison}:
\begin{align}
\Vert \mu_n - \pi_n \Vert_{TV}& \leq \Vert \mu_n - \bar R_n(\mu_{n-1})\Vert + \frac{2}{\mixing^2\log 3}\sum_{k=1}^{n-1}\tau_\mixing^{n-k-1}\frac{1}{\mixing^2}\Vert \mu_k - \bar R_{k}(\mu_{k-1})\Vert_{TV}
\end{align}
where we use that the Hilbert contraction coefficient is sub-multiplicative and can be bounded away from $1$ as a function of $\mixing$ (we denote it $\tau_\mixing$ this upper-bound given in \cref{prop:hilbert}). By denoting $\delta_n = \Vert \bar\mu_n - \bar{R}_n(\bar\mu_{n-1})\Vert_{TV}$ we obtain the result.
\end{proof}

We now prove \cref{lemma:bound}.

\begin{proof}
We apply the triangle inequality to $\Vert \pi^\nu_n - \hat\pi_0\Vert_{TV}$:
\begin{align}
    \Vert \pi^\nu_n - \hat\pi_0\Vert_{TV} \leq  \Vert \pi^\nu_n - \pi_n^{\hat \pi_0}\Vert_{TV} + \Vert \pi^{\hat\pi_0}_n - \hat\pi_n\Vert_{TV}
\end{align} and the result follows from \cref{prop:optimal-forgetting} and \cref{prop:lemma_bound}.
\end{proof}

\subsection{Bound with smooth, bounded approximations}

Let $\hat Q : E \times E\to\mathbb R^+$ and $\hat G: E \times F \to \mathbb R^+$ two bounded approximations of $Q$ and $G$. Given a sequence $(z_k)\in F^\mathbb N$, define the approximate non-negative kernel $\hat R_n(u, x) = \hat Q(u, x)\hat G(x, z_n)$ defined on $E\times E$. We introduce $(\hat\pi_n)$ the sequence of probability distributions computed using the recursion $\hat \pi_{n} = \hat R_n\hat\pi_{n-1}/R_n\hat\pi_{n-1}(E)=\barhat{R}_n(\hat\pi_{n-1})$.

\begin{proposition}\label{prop:bound-delta-n}
Let $\delta_n = \Vert \barhat{R}_n(\hat\pi_{n-1}) - \bar R_n(\hat\pi_{n-1})\Vert_{TV}$. Then,
\begin{align}
    \delta_n \leq & \frac{2}{\mixing\xi(E)}\left(\Vert G-\hat G\Vert_{L^\infty(E\times E)} + C_d\Vert \hat G\Vert_{L^\infty(E\times E)}\Vert Q-\hat Q\Vert_{L^\infty(E\times E)}\right)
\end{align}
where $C_d$ is a constant independent of $Q, G,\hat G, \hat Q$.
\end{proposition}

\begin{proof}
$\delta_n$ is the total variation distance between two probability distributions. Recall that 
\begin{align}
    \Vert \bar \mu - \bar \nu\Vert_{TV} \leq \frac{2}{\mu(E)}\Vert \mu - \nu\Vert_{TV}
\end{align}
for any $\mu, \nu$ positive, finite measures on $E$ and $\bar \mu, \bar \nu$ their normalized counterparts.
Thus,
\begin{align}
    \delta_n & \leq \frac{2}{R_n\hat\pi(E)}\Vert\hat R_n(\hat\pi_{n-1}) - R_n(\hat\pi_{n-1})\Vert_{TV}.
\end{align}

Recall that for any $\mu\in\mathcal P(E)$, $\mixing\xi(E) \leq R_n\mu(E)\leq \frac{1}{\mixing}\xi(E)$, which allows us to control the denominator:
\begin{align}
    \delta_n & \leq \frac{2}{\mixing\xi(E)}\Vert\hat R_n(\hat\pi_{n-1}) - R_n(\hat\pi_{n-1})\Vert_{TV}.
\end{align}

Because we will be able to bound the quality of approximation between $Q$ and $\hat Q$ (and between $G$ and $\hat G$), we split the above expression:
\begin{align}
\Vert\hat R_n(\hat\pi_{n-1}) - R_n(\hat\pi_{n-1})\Vert_{TV} \leq& \int\int \hat \pi_{n-1}(u)\vert Q(u, x)G(x, z_n) - \hat Q(u, x)\hat G(x, z_n)\vert dudx\\
& \leq \underbrace{\int\int Q(u, x) \hat \pi_{n-1}(u)\vert G(x, z_n) - \hat G(x, z_n)\vert dudx}_{A_n}\label{eq:delta-a} \\
& + \underbrace{\int\int \hat G(u, x) \hat \pi_{n-1}(u)\vert Q(u, x) - \hat Q(u, x)\vert dudx}_{B_n}.\label{eq:delta-b}
\end{align} 

First, let us bound $A_n$.
\begin{align}
A_n &= \int\int Q(u, x) \hat \pi_{n-1}(u)\vert G(x, y_n) - \hat G(x, y_n)\vert dudx\\
&\leq \Vert G-\hat G\Vert_{L^\infty(E\times E)}\int \hat\pi_{n-1}(u)\int Q(u,x)dxdu\\
&= \Vert G-\hat G\Vert_{L^\infty(E\times E)}.
\end{align}
where we used that $Q$ is a transition kernel (i.e. that $Q(u, E)=1$ for all $u\in E$) and that $\hat\pi_{n-1}$ is a distribution.

Second, let us bound $B_n$:
\begin{align}
B_n &= \int\int \hat G(x, z_n) \hat \pi_{n-1}(u)\vert Q(u, x) - \hat Q(u, x)\vert dudx\\
    &\leq \Vert Q-\hat Q\Vert_{L^\infty(E\times E)}\int \hat \pi_{n-1}(u)du \int \hat G(x, z_n)dx\\
    &\leq \textrm{Vol}(\Omega)\Vert \hat G\Vert_{L^\infty(E\times E)}\Vert Q-\hat Q\Vert_{L^\infty(E\times E)}
\end{align}
where we again used that $\hat \pi_{n-1}$ is a probability distribution.
\end{proof}

\subsection{Putting everything together}

We assume that $E=F$ without loss of generality (simply replace $D = 2d$ for $Q$ and $D=d+d^\prime$ for $G$). We choose $\beta > D/2$. Thus, here $\Omega = E\times E\subset \mathbb R^{D}$ where $D=2d$. 

Let $\gamma > 0$ and apply \cref{theorem:learning} with its parameter $\epsilon = \gamma$, Then we have that there exist $\hat G$ and $\hat Q$ and two constants $C_1, C_2$ such that 
$$\Vert \hat G-G\Vert_{L^\infty(\Omega)}\leq C_1 \Vert \sqrt{G}\Vert^2_{W^\beta_2(\Omega)}\gamma^{1-\frac{D}{2\beta}}, \quad \Vert \hat Q - Q \Vert_{L^\infty(\Omega)}\leq C_2\Vert\sqrt{Q}\Vert^2_{W^\beta_2(\Omega)}\gamma^{1-\frac{D}{2\beta}}$$
where $C_1$ and $C_2$ are independent of $Q,\hat Q, \hat G, G$.

By combining these inequalities with \cref{prop:bound-delta-n}, since $\xi(E) = 1$:
\begin{align}
    \delta_n \leq & \frac{2}{\mixing}\left(C\Vert \sqrt{G}\Vert^2_{W^\beta_2(\Omega)}\gamma^{1- \frac{D}{2\beta}} + C_d \Vert\hat G\Vert_{L^\infty(\Omega)}^2\Vert \sqrt{Q}\Vert^2_{W^\beta_2(\Omega)}\gamma^{1 - \frac{D}{2\beta}}\right).
\end{align}
Since $\Vert \hat G\Vert_{L^\infty(\Omega)} \leq \Vert G - \hat G \Vert_{L^\infty(\Omega)} + 2 \Vert \sqrt{G}\Vert_{L^\infty(\Omega)}^2 \leq \Vert G - \hat G \Vert_{L^\infty(\Omega)} + 2 \Vert \sqrt{G}\Vert^2_{W^\beta_2(\Omega)}$,
\begin{align*}
    \delta_n \leq & \frac{2}{\mixing}\left(C\Vert \sqrt{G}\Vert^2_{W^\beta_2(\Omega)}\gamma^{1- \frac{D}{2\beta}} + \left(C + C\epsilon^{1- \frac{D}{2\beta}}\right)\Vert\sqrt{G}\Vert^2_{W^\beta_2(\Omega)}\Vert\sqrt{Q}\Vert^2_{W^\beta_2(\Omega)}\gamma^{1 - \frac{D}{2\beta}}\right)
\end{align*}
Then, by upper-bounding the negligible terms, there exist two constants $C_1, C_2>0$ independent of $Q, \hat Q, G, \hat G$, such that :
\begin{align}
    \delta_n \leq & \underbrace{C_1\left(C_2 \Vert\sqrt{G}\Vert^2_{W^\beta_2(\Omega)} + C_2\Vert \sqrt{G}\Vert^2_{W^\beta_2(\Omega)}\Vert \sqrt{Q}\Vert^2_{W^\beta_2(\Omega)} \right)}_{C^\prime(Q, G)} \frac{\gamma^{1-\frac{D}{2\beta}}}{\sigma}
\end{align}
Note that $C^\prime$ only depends on parameters on $\Vert \sqrt{G}\Vert_{W^\beta_2(\Omega)}, \Vert \sqrt{Q}\Vert_{W^\beta_2(\Omega)}, d, \beta, \Omega$ and does not depend on $\hat{Q}, \hat{G}, \sigma, \gamma$.

We now apply \cref{theorem:learning}.
Let $\varepsilon > 0$ and $\delta > 0$. As a consequence of the development above, if $\gamma$ is chosen such that $\gamma = \left(\frac{\varepsilon}{C^\prime}\right)^{\frac{2\beta}{2\beta - D}}/\sigma$ then (1) $M, n$ correspond to the ones stated in the statement of the theorem (2)  with probability at least $1-6\delta$,
\begin{align}
    \delta_n \leq \frac{\varepsilon}{\sigma}
\end{align}
and then (3), 

\section{Computations on Generalized Gaussian PSD Models}\label{sec:proof-ops}

The stability properties of Generalized Gaussian PSD Models under probabilistic operations rely at a high-level on the fact that $Tr(AB)Tr(CD)=Tr(A\otimes C B \otimes D)$ and that if $A$ and $B$ are positive semi-definite matrices then so is $A\otimes B$.

\subsection{\textsc{Integral}}
\begin{proposition}[Integration of a Generalized Gaussian PSD Model]
   Let $f(x) = Tr(AB(x))$ with parameters $\left\lbrace A, C, (P_{ij}), (\mu_{ij})\right\rbrace$. Then, $Z =\int f(x)dx$ where 
   \begin{align}
       Z = Tr(A \circ \exp^\circ(C) \circ C(P))
   \end{align}
   where $\exp^\circ$ is the element-wise exponential map and $C(P)\in\mathcal S(\mathbb R^M)$ is decribed by $C(P)_{ij}=C(P_{ij})$.
\end{proposition}
We denote $\circ$ the Hadamard product.
\begin{proof}
The proof is clear by linearity of the trace.
\end{proof}
\begin{remark}[Computational complexity]
Because of the need to compute the determinant of $P_{ij}$ the computational complexity of the partial evaluation operation is $O(M^2 d!)$.
\end{remark}
\subsection{\textsc{PartialEval}}
\begin{proposition}[Partial evaluation of a Generalized Gaussian PSD Model]
   Let $f(x, y) = Tr(AB(x, y))$ with parameters $\left\lbrace A, C, (P_{ij}), (\mu_{ij})\right\rbrace$ and $y\in \mathbb R^d$. Then, $g(x):=f(x, y) = Tr(A^\prime B^\prime(x))$ with parameters $\left\lbrace A^\prime, C^\prime, (P^\prime_{ij}), (\mu^\prime_{ij})\right\rbrace$ where
   \begin{align}
       A^\prime &= A\\
       P^\prime_{ij} &= P_{ijxx}\\
       \mu^\prime_{ij} &= \mu_{ijx} + P_{ijxx}^{-1}P_{ijxy}\left(\mu_{ijy} - y\right)\\
        C_{ij}^\prime &= C_{ij} + \nu P_{xx}^{-1}\nu - yP_{yy}y + 2\mu_{x}P_{xy}y + 2\mu_yP_{yy}y-\mu P\mu
       \end{align}
where \begin{equation}
P = \left(\begin{array}{c c}
    P_{xx} & P_{xy}  \\
    P_{xy}^T & P_{yy}
\end{array} \right).
\end{equation}
\end{proposition}
\begin{proof}
We can compute $B^\prime(x)$ by expanding $C - \log(B(x, y)_{ij})$ for any $i,j$. Dropping the $i,j$ dependence:
\begin{align}
     \left\Vert P^{1/2}\left(\begin{bmatrix}x\\ y\end{bmatrix} - \mu \right)\right\Vert^2 =&  x^TP_{xx}x + 2y P_{xy}^Tx + yP_{yy}y - 2\mu P\begin{bmatrix}x\\ y\end{bmatrix}  + \mu P \mu\\
    =&  x^TP_{xx}x + 2y P_{xy}^Tx + yP_{yy}y - 2\mu_x P_{xx}x -2\mu_xP_{xy}y -2\mu_yP_{xy}x -2\mu_yP_{yy}y + \mu P \mu\\
    =&  x^TP_{xx}x -2(\underbrace{\mu_xP_{xx}+\mu_y P_{xy} -yP_{xy}^T}_\nu)x + yP_{yy}y -2\mu_xP_{xy}y -2\mu_yP_{yy}y + \mu P \mu\\
    =&  \left\Vert P_{xx}^{1/2}\left(x - P_{xx}^{-1}\nu\right)\right\Vert^2 - \nu P_{xx}^{-1}\nu + yP_{yy}y -2\mu_xP_{xy}y -2\mu_yP_{yy}y + \mu P \mu.
\end{align}
\end{proof}
\begin{remark}[Computational complexity]
Because of the need to compute the inverse of $P_{xx}$, the computational complexity of the partial evaluation operation is $O(M^2d_x^3)$.
\end{remark}

\subsection{\textsc{Marginalization}}
\begin{proposition}[Marginalization of a Generalized Gaussian PSD Model]
   Let $f(x, y) = Tr(AB(x, y))$ with parameters $\left\lbrace A, C, (P_{ij}), (\mu_{ij})\right\rbrace$. Then, $h(x):=\int f(x, y)dy = Tr(A^\prime B^\prime(x))$ with parameters $\left\lbrace A^\prime, C^\prime, (P^\prime_{ij}), (\mu^\prime_{ij})\right\rbrace$ where
   \begin{align}
       A^\prime &= A\\
        P_{ij}^\prime &= \left(\left[P^{-1}_{ij}\right]_{xx}\right)^{-1}\\
       \mu^\prime_{ij} &= \left[\mu_{ij}\right]_{x}\\
          C^\prime_{ij} &= C_{ij} + \log(C(P_{ij})) - \log(C(P^ \prime_{ij}))
       \end{align}
\end{proposition}
\begin{proof}
We compute the integral component-wise, denoting $\Sigma = (2P)^{-1}$:

\begin{align}
    \int B(x, y)_{ij}dy &= e^{C_{ij}}\int \exp\left(- \frac{\left\Vert \sqrt{2}P_{ij}^{1/2}\left[\begin{pmatrix}x\\ y\end{pmatrix} - \mu_{ij}\right] \right\Vert^2}{2}\right)dy\\
    &= e^{C_{ij}}\int \exp\left(- \frac{\left\Vert \Sigma^{-1/2}\left[\begin{pmatrix}x\\ y\end{pmatrix} - \mu_{ij}\right] \right\Vert^2}{2}\right)dy\\
    &= e^{C_{ij}}\sqrt{(2\pi)^{d_x + d_y}\vert \Sigma \vert}\int \frac{1}{\sqrt{(2\pi)^{d_x + d_y}\vert \Sigma \vert}}\exp\left(- \frac{\left\Vert \Sigma^{-1/2}\left[\begin{pmatrix}x\\ y\end{pmatrix} - \mu_{ij}\right] \right\Vert^2}{2}\right)dy\\
    &= e^{C_{ij}}\sqrt{(2\pi)^{d_x + d_y}\vert \Sigma\vert}\frac{1}{\sqrt{(2\pi)^{d_x}\vert \left[\Sigma\right]_{xx}\vert}}\exp\left(- \frac{\left\Vert [\Sigma]_{xx}^{-1/2}\left( x - \left[\mu_{ij}\right]_x\right)\right\Vert^2}{2}\right)\\
    &= e^{C_{ij}}\frac{C_{P_{ij}}}{C_{P_{ij}^\prime}}\exp\left(-\left\Vert {P_{ij}^\prime}^{1/2}\left( x - \left[\mu_{ij}\right]_x\right)\right\Vert^2\right)\\
    &= e^{C^\prime_{ij}}\exp\left(-\left\Vert {P_{ij}^\prime}^{1/2}\left( x - \mu_{ij}^\prime\right)\right\Vert^2\right)
\end{align}

where $C^\prime_{ij} = C_{ij} + \log(C_{P_{ij}}) - \log(C_{P^\prime_{ij}})$, $P^\prime_{ij} = \left(\left[P_{ij}^{-1}\right]_{xx}\right)^{-1}$ and $\mu_{ij}^\prime = \left[\mu_{ij}\right]_x$.
\end{proof}

\begin{remark}[Computational complexity]
Because of the need to compute the determinant of $P$ as well as invert it, the computational complexity of the partial evaluation operation is $O(M^2\max(d!, d^3))$.
\end{remark}
\subsection{\textsc{Product}}
\begin{proposition}[Product of two Generalized Gaussian PSD Models]
   Let $f(x, y) = Tr(AB(x, y))$ a generalized PSD model of order $M$ with parameters $\left\lbrace A, C, (P_{ij}), (\mu_{ij})\right\rbrace$.
   Let $g(x, y) = Tr(\tilde A \tilde B(x, y))$ a generalized PSD model of order $m$ with parameters $\left\lbrace \tilde A, \tilde C, (\tilde P_{kl}), (\tilde\mu_{kl})\right\rbrace$.
   
   Then, $h(x):=f(x)g(x)= Tr(A^\prime B^\prime(x))$ is a generalized PSD model of order $Mm$ with parameters $\left\lbrace A^\prime, C^\prime, (P^\prime_{ij}), (\mu^\prime_{ij})\right\rbrace$ where
    \begin{align}
    A^\prime &= A \otimes \tilde A\\
    P_{ijkl}^\prime &= \left[\begin{array}{ccc}
    P_{ijxx} & P_{ijxy}^T & 0 \\
    P_{ijxy} &P_{ijyy} + \tilde P_{klyy}  & \tilde P_{klyz} \\
    0 & \tilde P_{klyz} & \tilde P_{klzz}
    \end{array}\right]\\
    \mu^\prime_{ijkl} &= {P^\prime_{ijkl}}^{-1}\hat\mu_{ijkl}\\
    C_{ijkl} &= C^f_{ij} + C^g_{kl} +\hat\mu_{ijkl}P_{ijkl}\hat\mu_{ijkl} - \mu_{ij}P_{ij}\mu_{ij} - \tilde\mu_{kl}\tilde P_{kl}\tilde\mu_{kl}
\end{align}

where \begin{equation}
    \hat\mu_{ijkl} =\begin{pmatrix}
    {P_{ij}}\mu_{ij}\\
    0\end{pmatrix} + 
    \begin{pmatrix}
     0 \\
    {\tilde P_{kl}}\tilde\mu_{kl}
    \end{pmatrix} = 
    \begin{pmatrix}
    [{P_{ij}}\mu_{ij}]_x\\
    [{P_{ij}}\mu_{ij}]_{y} + [{\tilde P_{kl}}\tilde\mu_{kl}]_{y}\\
    [{\tilde P_{kl}}\tilde\mu_{kl}]_{z}
     \end{pmatrix}
     \end{equation}
\end{proposition}

\begin{proof}
Notice that:
    \begin{align}
        \log\left(B_f(x, y) \otimes B_g(y, z)_{ijkl}\right)& = C^f_{ij} + C^g_{kl} - \left\Vert {P_{ij}}^{1/2}\left[\begin{pmatrix}x\\ y\end{pmatrix}- \mu_{ij}\right]\right\Vert^2 - \left\Vert {\tilde P_{kl}}^{1/2}\left[\begin{pmatrix}x \\ y \end{pmatrix}- \tilde\mu_{kl}\right]\right\Vert^2
    \end{align}
Let us compute the following term by computing the square:
    \begin{align}
     &\left\Vert {P_{ij}}^{1/2}\left[\begin{pmatrix}x\\y\end{pmatrix}- \mu_{ij}\right]\right\Vert^2 + \left\Vert {\tilde P_{kl}}^{1/2}\left[\begin{pmatrix}y\\ z\end{pmatrix}- \tilde\mu_{kl}\right]\right\Vert^2 \\
     &= \begin{pmatrix}x\\ y\\ z\end{pmatrix}^TP_{ijkl}\begin{pmatrix}x\\ y\\ z\end{pmatrix}- 2\hat\mu_{ijkl}^T\begin{pmatrix}x\\ y\\ z\end{pmatrix} +\mu_{ij}P_{ij}\mu_{ij} + \tilde\mu_{kl}\tilde P_{kl}\tilde\mu_{kl}\\
     &= \left\Vert P_{ijkl}^{1/2}\left(\begin{pmatrix}x\\ y\\ z\end{pmatrix} - P_{ijkl}^{-1}\hat\mu_{ijkl}\right) \right\Vert^2 - \hat\mu_{ijkl}P_{ijkl}^{-1}\hat\mu_{ijkl}+\mu_{ij}P_{ij}\mu_{ij} + \tilde\mu_{kl}\tilde P_{kl}\tilde\mu_{kl}\\
    \end{align}

    where
    \begin{equation}
    P_{ijkl} = \left[\begin{array}{ ccc }
    {P_{ij}}_{xx} & {P_{ij}}_{xy}^T & 0 \\
    {P_{ij}}_{xy} &P_{ijyy} + \tilde P_{klyy}  & \tilde P_{klyz} \\
    0 & \tilde P_{klyz} & \tilde P_{klzz}
    \end{array}\right],
\end{equation}
    \begin{equation}
    \hat\mu_{ijkl} = \begin{pmatrix}
    {P_{ij}}\mu_{ij}\\
    0\end{pmatrix} + 
    \begin{pmatrix}
     0 \\
    {\tilde P_{kl}}\tilde\mu_{kl}
    \end{pmatrix} = 
    \begin{pmatrix}
    [{P_{ij}}\mu_{ij}]_x\\
    [{P_{ij}}\mu_{ij}]_{y} + [{\tilde P_{kl}}\tilde\mu_{kl}]_{y}\\
    [{\tilde P_{kl}}\tilde\mu_{kl}]_{z}
     \end{pmatrix}
\end{equation}

So 
\begin{equation}
    (B_f(x, y) \otimes B_g(y, z))_{ijkl} =  \exp\left(C_{ijkl}-\left\Vert P_{ijkl}^{1/2}\left((x, y, z) - \mu_{ijkl}\right)\right\Vert^2\right)
\end{equation}
where
\begin{equation}
    \mu_{ijkl} = P_{ijkl}^{-1}\hat\mu_{ijkl}
\end{equation}
\begin{equation}
    C_{ijkl} = C^f_{ij} + C^g_{kl} +\hat\mu_{ijkl}P_{ijkl}^{-1}\hat\mu_{ijkl} -  \mu_{ij}P_{ij}\mu_{ij} - \tilde\mu_{kl}\tilde P_{kl}\tilde\mu_{kl}
\end{equation}
\end{proof}
\begin{remark}[Computational complexity]
Because of the need to compute the inverse of $P^\prime$ the computational complexity of the product operation between models of order $M$ and $m$ is $O(M^2m^2d^3)$.
\end{remark}

\subsection{Proof of \cref{theorem:kalman}}\label{sec:proof_kalman}
\begin{proof} 
Let $P = L^T\Sigma^{-1}L$ where $L = (F ~-I)$ and $P_\lambda = P + \lambda I$.

We have
\begin{align}
    -\log p(y|x) &= -C_\Sigma + \Vert \Sigma^{-1/2}(Fx + b -y)\Vert^2 = -C_\Sigma + \Vert \Sigma^{-1/2}(Lu + b)\Vert^2\\
                 &= -C_\Sigma + uL^T\Sigma^{-1}Lu - 2b^T\Sigma^{-1}Lu + b^T\Sigma^{-1}b
\end{align}

If we define $\mu = P_{\lambda}^{-1}\beta$ where $\beta = L^T\Sigma^{-1}b$ then,
\begin{align}
    -\log p(y|x) & = -C + \Vert P_\lambda^{1/2}\left(u - \mu\right)\Vert^2\\
                 & = -C + u^T P_\lambda u - 2 \mu^T P_\lambda u + \mu^T P_\lambda \mu\\
                 & = -C + u^TPu + \lambda \Vert u \Vert^2 -2 \beta^Tu + \beta^TP_\lambda^{-1}P_\lambda P_\lambda^{-1}\beta\\
                 & = -C + u^TPu + \lambda \Vert u \Vert^2 -2 \beta^Tu + \beta^TP_\lambda^{-1}\beta\\
\end{align}

And so,
\begin{align}
    - \log(p(y|x)/\hat p(y/x)) &= -C_\Sigma + uPu - 2\beta^Tu + b^T\Sigma^{-1}b + C - u^TPu - \lambda \Vert u \Vert^2 +2 \beta^Tu - \beta^TP_\lambda^{-1}\beta\\
   &= C -C_\Sigma+ b^T\Sigma^{-1}b - \lambda \Vert u \Vert^2 - \beta^TP_\lambda^{-1}\beta
\end{align}

Using Woodbury,
\begin{align}
    b^T \Sigma^{-1}b - \beta^TP_\lambda^{-1}\beta &= b^T\left(\Sigma^{-1} - \Sigma^{-1}L\left[ L^T\Sigma^{-1}L + \lambda I\right]^{-1}L^T-\Sigma^{-1}\right)b\\
    &=\lambda b^T\left( \lambda \Sigma + LL^T\right)^{-1}b
\end{align}

and

\begin{align}
    - \log(p(y|x)/\hat p(y/x)) &= C -C_\Sigma + \lambda b^T\left( \lambda \Sigma + LL^T\right)^{-1}b- \lambda \Vert u \Vert^2 
\end{align}

With $C = C_\Sigma - \lambda b^T\left( \lambda \Sigma + LL^T\right)^{-1}b$, $\frac{p(y|x)}{\hat p(x, y)} = e^{\lambda \Vert u \Vert^2}$ and the result follows.
\end{proof}

\subsection{Learning Generalized Gaussian PSD Models}\label{sec:learning-general}
From an approximation perspective, a Generalized Gaussian PSD Model is a Gaussian PSD Model in which one can optimize the anchor points $\tilde x$ and precision matrices $P$ of each kernel function. In the case of approximating transition kernels, this can yield significant improvements in model order. Indeed, a transition kernel $Q(u, x)$ is a conditional probability distribution which depends in which the probability of the value $x$ depends on the value $u$. This dependence is encoded in the combination of kernel evaluations but not in the kernel evaluations themselves.

To approximate a function $f$ with a Generalized Gaussian PSD Model, we implicitly approximate the square-root of $f$ using a Gaussian Linear Model:
\begin{align}\label{eq:non-convex}
    \min_{\hat g \in \mathcal G_M} \frac{1}{n}\sum_{i=1}^n \left\vert f(x_i)- \hat g(x_i)^2\right\vert^2,
\end{align}

where $x_i$ are sampled or chosen on a grid, and $\mathcal G_M = \lbrace \sum_{j=1}^M\alpha_i k_{P_i}(x, \mu_i) ~\vert~ \mu_i \in \mathbb R^d, P = R_i^\top R_i, R_i \in\mathbb R^{d\times d}, \alpha_i \in \mathbb R \text{ for } 1\leq j\leq M \rbrace$. \cref{eq:non-convex} is a smooth, non-convex problem which can be solved approximately using off-the-shelf solver like L-BFGS \citep{lbfgs}.

In practice, we initialize the model by placing $\mu_i$ is regions where $f(\mu_i)$ is large. In the case where $f(u, x)$ is a transition kernel $Q(u, x)$, one strategy is to chose $[\mu_i]_u$ on a grid (or sampled uniformly) and then choose $[\mu_i]_v$ such that $f(\mu_i) = \sup_x f([\mu_i]_u, v)$. This is particularly interesting when $Q$ is a non-linear Gaussian model $Q(u, x) \propto e^{-\Vert\Sigma^{-1/2}(x - h(u))\Vert^2}$ for some non-linear transition model $h$.
\end{document}